\documentclass[preprint,12pt]{alt2024} %

\title[Learning bounded-degree polytrees with known skeleton]{Learning bounded-degree polytrees with known skeleton}

\usepackage{times}
\usepackage{authblk}
\author[1]{Davin Choo\thanks{Equal contribution}}
\author[2]{Joy Qiping Yang$^*$}
\author[3]{Arnab Bhattacharyya}
\author[4]{Clément L. Canonne}
\affil[1,2,3]{National University of Singapore}
\affil[2,4]{University of Sydney}

\usepackage{algorithm}
\DontPrintSemicolon
\LinesNumbered
\RestyleAlgo{ruled}
\SetKwComment{Comment}{$\triangleright$ }{}
\SetKwInput{KwData}{Input}
\SetKwInput{KwResult}{Output}
\SetKwRepeat{Do}{do}{while}

\usepackage{enumerate}
\usepackage[nointegrals]{wasysym}
\usepackage{bbm}
\usepackage{amsmath}
\usepackage{pifont}
\usepackage{booktabs}
\usepackage{xcolor}
\definecolor{darkgreen}{rgb}{0,0.5,0}
\usepackage{hyperref}
\hypersetup{
    unicode=false,          %
    colorlinks=true,        %
    linkcolor=red,          %
    citecolor=darkgreen,    %
    filecolor=magenta,      %
    urlcolor=blue           %
}
\usepackage[capitalize, nameinlink]{cleveref}

\usepackage{thmtools}
\usepackage{thm-restate}
\newtheorem{mylemma}[theorem]{Lemma}
\newtheorem{mycorollary}[theorem]{Corollary}
\newtheorem{claim}[theorem]{Claim}
\newtheorem{assumption}[theorem]{Assumption}

\usepackage{tikz}
\usetikzlibrary{calc, graphs, graphs.standard, shapes, arrows, arrows.meta, positioning, decorations.pathreplacing, decorations.markings, decorations.pathmorphing, fit, matrix, patterns, shapes.misc, tikzmark}

\newcommand{\bX}{\mathbf{X}}
\newcommand{\bY}{\mathbf{Y}}

\newcommand{\cN}{\mathcal{N}}
\newcommand{\cO}{\mathcal{O}}

\newcommand{\cX}{\mathcal{X}}
\newcommand{\cY}{\mathcal{Y}}
\newcommand{\cZ}{\mathcal{Z}}

\newcommand{\pd}{P}
\newcommand{\eps}{\varepsilon}
\newcommand{\skel}{\mathrm{skel}}

\newcommand{\Bern}{\mathrm{Bern}}
\newcommand{\kl}{d_{\mathrm{KL}}}

\newcommand{\sqhell}{d_{\mathrm{H}}^2}

\definecolor{DSred}{rgb}{1,0,0}

\usepackage{texmacs}

\allowdisplaybreaks
\newcommand{\ignore}[1]{}

\begin{document}

\maketitle

\begin{abstract}
We establish finite-sample guarantees for efficient proper learning of bounded-degree \emph{polytrees}, a rich class of high-dimensional probability distributions and a subclass of Bayesian networks, a widely-studied type of graphical model.
Recently, \cite{DBLP:journals/corr/abs-2011-04144} obtained finite-sample guarantees for recovering tree-structured Bayesian networks, i.e., 1-polytrees.
We extend their results by providing an efficient algorithm which learns $d$-polytrees in polynomial time and sample complexity for any bounded $d$ when the underlying undirected graph (skeleton) is known.
We complement our algorithm with an information-theoretic sample complexity lower bound, showing that the dependence on the dimension and target accuracy parameters are nearly tight.
\end{abstract}


\section{Introduction}
\label{sec:intro}

Distribution learning, or density estimation, is the task of obtaining a good estimate of an unknown underlying probability distribution $\pd$ from observational samples.
Understanding which classes of distributions could be or could not be learnt efficiently, in terms of \emph{sample} (data) and \emph{computational} (time) complexities, is a fundamental problem in both computer science and statistics.

{\em Bayesian networks} (or {\em Bayes nets} in short) represent a class of high-dimensional distributions that can be explicitly described by how each variable is generated sequentially in a directed fashion.
Being interpretable, Bayes nets have been used to model beliefs in a wide variety of domains (e.g.\ see \cite{jensen2007bayesian,DBLP:books/daglib/0023091} and references therein).
A fundamental problem in computational learning theory is to identify families of Bayes nets which can be learned efficiently from observational data.
Formally, a Bayes net is a probability distribution $\pd$, defined over some directed acyclic graph (DAG) $G = (V,E)$ on $|V| = n$ nodes that factorizes according to $G$ (i.e. Markov with respect to $G$) in the following sense: $\pd(v_1, \ldots, v_n) = \prod_{v_1, \ldots, v_n} \pd(v \mid \pi(v))$, where $\pi(v) \subseteq V$ are the parents of $v$ in $G$.
While it is well-known that given the DAG structure of a Bayes net, there exists sample-efficient algorithms\footnote{In terms of computational efficiency, one can efficiently learn the distribution by following the DAG structure to learn each local probability table. See \cite[Section 6]{DBLP:journals/siamcomp/BhattacharyyaGPTV23} for a formal analysis of such an approach; there is an exponential dependency on the number of parents but this is unavoidable.} that output good hypotheses \citep{DBLP:journals/ml/Dasgupta97, BhattacharyyaGMV20}, there is no known computationally efficient algorithms for obtaining the DAG of a Bayes net.
In fact, it has long been understood that Bayes net structure learning is computationally expensive in general \citep{DBLP:journals/jmlr/ChickeringHM04}.
However, these hardness results fall short when the goal is learning the distribution $\pd$ in the probabilistically approximately correct (PAC) \citep{DBLP:conf/stoc/Valiant84} sense (with respect to, say, KL divergence or total variation distance), rather than trying to recover an exact graph structure from the Markov equivalence class of $\pd$.
That is, prior hardness results on \emph{exact} recovery do not carry over to \emph{approximate} recovery in the PAC sense.

{\em Polytrees} are a subclass of Bayesian networks where the undirected graph underlying the DAG is a forest, i.e.\ there are no cycles in the undirected graph obtained by ignoring edge directions.
A polytree with maximum in-degree $d$ is also known as a $d$-polytree.
With an infinite number of samples, one can recover the DAG of a non-degenerate polytree in the equivalence class with the Chow--Liu algorithm \citep{DBLP:journals/tit/ChowL68} and some additional conditional independence tests \citep{DBLP:journals/corr/abs-1304-2736}.
However, this algorithm does \emph{not} work in the finite sample regime. The only known result for learning polytrees with finite sample guarantees is for 1-polytrees \citep{DBLP:journals/corr/abs-2011-04144}.
Furthermore, in the agnostic setting, when the goal is to find the closest polytree distribution to an arbitrary distribution $\pd$, the learning problem becomes NP-hard \citep{DBLP:conf/uai/Dasgupta99}.

Here, we investigate what happens when the given distribution is a $d$-polytree, for $d > 1$.
\emph{Are $d$-polytrees computationally hard to learn in the realizable PAC-learning setting?}
One motivation for studying polytrees is due to a recent work of \cite{gao2021efficient} which shows that polytrees are easier to learn than general Bayes nets due to the underlying graph being a tree, allowing typical causal assumptions such as faithfulness to be dropped when designing efficient learning algorithms.

\paragraph{Contributions.}
We give a sample-efficient algorithm for proper Bayes net learning in the realizable setting, when provided with the ground truth skeleton (i.e., the underlying forest).
Crucially, our result does not require any distributional assumptions such as strong faithfulness, etc.
We also give information-theoretic sample complexity lower bounds that hold even when the ground truth skeleton is known and given to us.

\begin{theorem}
\label{theo:main}
Consider a discrete distribution $\pd$ on $n$ variables, each with alphabet $\Sigma$, defined on a polytree $G^*$ with an unknown maximum in-degree $d^*$.
Given $m$ samples from $\pd$, accuracy parameter $\eps>0$, failure probability $\delta$, the skeleton of $G^*$, and an in-degree upper bound $d \geq d^*$, there exists an algorithm that outputs a $d$-polytree distribution $\hat{\pd}$ such that $\kl(\pd \;\|\; \hat{\pd}) \leq \eps$ with success probability at least $1 - \delta$, as long as
\[
m
= \tilde{\Omega} \!\left( \frac{n \cdot |\Sigma|^{d+1}}{\eps}  \log \frac{1}{\delta} \right) \;.
\]
Moreover, the algorithm runs in time polynomial in $m$, $|\Sigma|^d$, and $n^d$.
\end{theorem}

We remark that our result holds when given only an upper bound $d$ on the true in-degree $d^*$.
In particular, our result yields a sample complexity upper bound of $\tilde{O}(n/\eps)$ for learning $O(1)$-polytrees with constant $|\Sigma|$ and $d$.
Note that this dependence on the dimension $n$ and the accuracy parameter $\eps$ is optimal, up to logarithmic factors: indeed, we establish in~\cref{thm:lower-bound-for-n-variables} an $\Omega(n/\eps)$ sample complexity lower bound for this question, even for $d=2$ and $|\Sigma|=2$.\footnote{We remark that~\cite[Theorem~7.6]{DBLP:journals/corr/abs-2011-04144} implies an $\Omega(\frac{n}{\eps}\log\frac{n}{\eps})$ sample complexity lower bound for the analogous question when the skeleton is unknown and $d=1$.}

We also state sufficient distributional conditions that enable recovery of the ground truth skeleton.
Informally, we require that the data processing inequality hold in a strong sense with respect to the edges in the skeleton.
Under these conditions, combining with our main result in \cref{theo:main}, we obtain a polynomial-time PAC algorithm to learn bounded-degree polytrees from samples.

\subsection{Overview of algorithm}

Our algorithm is designed with KL divergence in mind; see \cref{eq:mutual_information_decom_chou_liu}.
Since there are efficient algorithms for estimating the parameters of a Bayes net with in-degree $d$ once a close-enough graph $\hat{G}$ is recovered \citep{DBLP:journals/ml/Dasgupta97, BhattacharyyaGMV20}, it suffices to find a good approximation of the underlying DAG $G^*$.
For any distribution $\pd$ and DAG $G$, let $\pd_G$ be defined as the projection of $P$ unto a DAG $G$; see \cref{ssec:prob_dist_def} for the formal definition.
Then, for a distribution $\pd$ that is Markov with respect to a DAG $G^*$, the quality (in terms of KL divergence) of approximating $G^*$ with $G$ is $\kl(\pd, \pd_{G}) = \kl(\pd_{G^{*}}, \pd_{G}) = \sum_{v \in V} I(v; \pi_{G^{*}}(v)) - I(v; \pi_{G}(v))$, where $I(\cdot;\cdot)$ refers to {\em mutual information (MI)}, and $\pi_{G}(v)$ is the set of parents of $v$ in $G$.
When the true skeleton $\skel(G^{*})$ is given to us in advance, what remains is to orient each edge.
As such, given error parameter $\eps > 0$ and upper bound on in-degree $d$, the goal of our algorithm is to judiciously orient the edges of $\skel(G^{*})$ such that $\kl(\pd_{G^{*}}, \pd_{G})$ is at most $\eps$ while ensuring that every vertex has at most $d$ incoming edges.

Our algorithm relies on estimating MI and conditional MI (CMI) terms involving subsets of variables.
A na\"ive approach of estimating these terms additively would incur an unnecessary overhead on the sample complexity.
One of our technical contributions is to show that it suffices to have access to a \emph{tester} that can distinguish between a CMI term being $0$ or at least $\eta$, for some threshold $\eta >0$.
As shown in \cite{DBLP:journals/corr/abs-2011-04144}, the sample complexity for testing (see \cref{cor:CMI_tester}) is an $O(\eta)$ factor smaller than that for estimating the CMI up to additive error of $\pm$ $\eta/2$.

Our algorithm works in three phases.
In the first phase, we orient ``strong v-structures''; see \cref{sec:prelims} for a precise definition.
In the second phase, we locally check if an edge is ``forced'' to orient in a specific direction.
In the third phase, we orient the remaining unoriented edges as a 1-polytree.
Throughout the algorithm, we do \emph{not} unorient edges as we will be able to argue that any orientations performed by the first two phases are guaranteed to respect the orientations of the underlying causal graph from which we draw samples from\footnote{Note that for some distributions there could be more than one ground truth graph, e.g.\ when the Markov equivalence class has multiple graphs. In such situations, for analysis purposes, we are free to choose any graph that the underlying distribution is Markov with respect to as the ``ground truth''.}.

To explain the intuition behind the first two phases, consider the example of a path on 3 vertices $u - v - w$; see \cref{fig:running-example} for a slightly more sophisticated example.
If $u \to v \gets w$, then $u$ and $w$ are \emph{dependent} given $v$.
Otherwise, $u$ and $w$ are \emph{independent} given $v$ since $G$ is a polytree.
That is, one would expect $I(u;w \mid v)$ to be large if and only if $u \to v \gets w$ was a v-structure.
If it is indeed the case that $I(u;w \mid v)$ is ``large'', then this would be detected by our finite-sample CMI test (i.e.\ $u \to v \gets w$ was ``strong'') and so we orient $u \to v$ and $w \to v$ in Phase 1.
Now, after Phase 1, some edges in the graph would be partially oriented; say, we have $u \to v - w$ after Phase 1.
If $u \to v \to w$ was the ground truth, then $I(u;w \mid v) = 0$ and our tester will detect this term as ``small''.
If $u \to v \gets w$ was the ground truth, then $I(u;w) = 0$ and our tester will detect this term as ``small''.
Via the contrapositive of the previous two statements, if $I(u;w \mid v)$ or $I(u;w)$ is ``large'', then we know a specific orientation of the edge $v - w$.
We may also leave $v - w$ unoriented if neither term was ``large''.
Another form of ``forced orientation'' is due to the given upper bound $d$ on the number of parents any vertex can have: we should point all remaining incident unoriented edges \emph{away} from a vertex $v$ whenever $v$ already has $d$ incoming arcs.
For example, if $d = 1$, then we would immediately orient $v \to w$ if we observe $u \to v - w$ after Phase 1.
Given the above intuition, any edge that remains unoriented till the Phase 3 must have been ``flexible'' in the sense that it could be oriented either way.
In fact, we later show that ``not too much error'' has been incurred if the edge orientations from the final phase only increases the incoming degrees of any vertex by at most one.

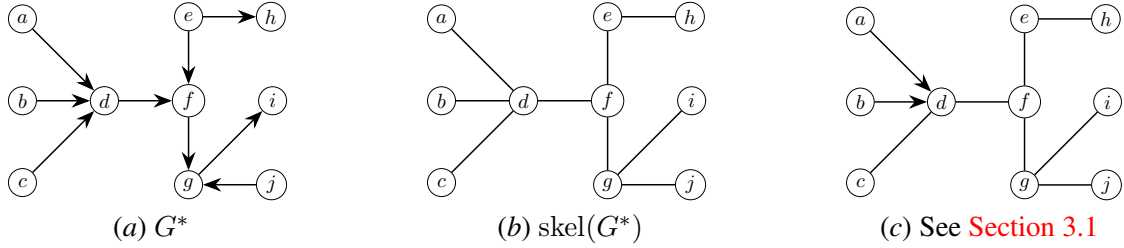
\begin{figure}[tb]
\centering
\begin{subfigure}[$G^*$][b]{
    \centering
    \resizebox{0.25\linewidth}{!}{\resizebox{0.7\linewidth}{!}{%
\begin{tikzpicture}
    \node[draw, circle, minimum size=15pt, inner sep=2pt] at (0,0) (d) {$d$};
    \node[draw, circle, minimum size=15pt, inner sep=2pt, left=of d] (b) {$b$};
    \node[draw, circle, minimum size=15pt, inner sep=2pt, above=of b] (a) {$a$};
    \node[draw, circle, minimum size=15pt, inner sep=2pt, below=of b] (c) {$c$};
    \node[draw, circle, minimum size=15pt, inner sep=2pt, right=of d] (f) {$f$};
    \node[draw, circle, minimum size=15pt, inner sep=2pt, above=of f] (e) {$e$};
    \node[draw, circle, minimum size=15pt, inner sep=2pt, below=of f] (g) {$g$};
    \node[draw, circle, minimum size=15pt, inner sep=2pt, right=of e] (h) {$h$};
    \node[draw, circle, minimum size=15pt, inner sep=2pt, right=of f] (i) {$i$};
    \node[draw, circle, minimum size=15pt, inner sep=2pt, right=of g] (j) {$j$};

    \draw[thick, -{Stealth[scale=1.5]}] (a) -- (d);
    \draw[thick, -{Stealth[scale=1.5]}] (b) -- (d);
    \draw[thick, -{Stealth[scale=1.5]}] (c) -- (d);
    \draw[thick, -{Stealth[scale=1.5]}] (d) -- (f);
    \draw[thick, -{Stealth[scale=1.5]}] (e) -- (f);
    \draw[thick, -{Stealth[scale=1.5]}] (f) -- (g);
    \draw[thick, -{Stealth[scale=1.5]}] (e) -- (h);
    \draw[thick, -{Stealth[scale=1.5]}] (g) -- (i);
    \draw[thick, -{Stealth[scale=1.5]}] (j) -- (g);
\end{tikzpicture}
}}
}
\end{subfigure}
\hfill
\begin{subfigure}[$\skel(G^*)$][b]{
    \centering
    \resizebox{0.25\linewidth}{!}{\resizebox{0.7\linewidth}{!}{%
\begin{tikzpicture}
    \node[draw, circle, minimum size=15pt, inner sep=2pt] at (0,0) (d) {$d$};
    \node[draw, circle, minimum size=15pt, inner sep=2pt, left=of d] (b) {$b$};
    \node[draw, circle, minimum size=15pt, inner sep=2pt, above=of b] (a) {$a$};
    \node[draw, circle, minimum size=15pt, inner sep=2pt, below=of b] (c) {$c$};
    \node[draw, circle, minimum size=15pt, inner sep=2pt, right=of d] (f) {$f$};
    \node[draw, circle, minimum size=15pt, inner sep=2pt, above=of f] (e) {$e$};
    \node[draw, circle, minimum size=15pt, inner sep=2pt, below=of f] (g) {$g$};
    \node[draw, circle, minimum size=15pt, inner sep=2pt, right=of e] (h) {$h$};
    \node[draw, circle, minimum size=15pt, inner sep=2pt, right=of f] (i) {$i$};
    \node[draw, circle, minimum size=15pt, inner sep=2pt, right=of g] (j) {$j$};

    \draw[thick] (a) -- (d);
    \draw[thick] (b) -- (d);
    \draw[thick] (c) -- (d);
    \draw[thick] (d) -- (f);
    \draw[thick] (e) -- (f);
    \draw[thick] (f) -- (g);
    \draw[thick] (e) -- (h);
    \draw[thick] (g) -- (i);
    \draw[thick] (j) -- (g);
\end{tikzpicture}
}}
}
\end{subfigure}
\hfill
\begin{subfigure}[See \cref{sec:algo}][b] {
    \centering
    \resizebox{0.25\linewidth}{!}{\resizebox{0.3\linewidth}{!}{%
\begin{tikzpicture}
    \node[draw, circle, minimum size=15pt, inner sep=2pt] at (0,0) (d) {$d$};
    \node[draw, circle, minimum size=15pt, inner sep=2pt, left=of d] (b) {$b$};
    \node[draw, circle, minimum size=15pt, inner sep=2pt, above=of b] (a) {$a$};
    \node[draw, circle, minimum size=15pt, inner sep=2pt, below=of b] (c) {$c$};
    \node[draw, circle, minimum size=15pt, inner sep=2pt, right=of d] (f) {$f$};
    \node[draw, circle, minimum size=15pt, inner sep=2pt, above=of f] (e) {$e$};
    \node[draw, circle, minimum size=15pt, inner sep=2pt, below=of f] (g) {$g$};
    \node[draw, circle, minimum size=15pt, inner sep=2pt, right=of e] (h) {$h$};
    \node[draw, circle, minimum size=15pt, inner sep=2pt, right=of f] (i) {$i$};
    \node[draw, circle, minimum size=15pt, inner sep=2pt, right=of g] (j) {$j$};

    \draw[thick, -{Stealth[scale=1.5]}] (a) -- (d);
    \draw[thick, -{Stealth[scale=1.5]}] (b) -- (d);
    \draw[thick] (c) -- (d);
    \draw[thick] (d) -- (f);
    \draw[thick] (e) -- (f);
    \draw[thick] (f) -- (g);
    \draw[thick] (e) -- (h);
    \draw[thick] (g) -- (i);
    \draw[thick] (j) -- (g);
\end{tikzpicture}
}}
}
\end{subfigure}
\caption{
3-polytree example where $I(a; b,c) = I(b; a,c) = I(c; a,b) = 0$ due to deg-3 v-structure centered at $d$.
By \cref{cor:CMI_tester}, $I(a;f \mid d) = 0$ implies $\hat{I}(a; f \mid d) \leq C \cdot \eps$, and so we will \emph{not} detect $a \to d \to f$ erroneously as a strong deg-2 v-structure $a \to d \gets f$.
}
\label{fig:running-example}
\end{figure}

\subsection{Overview of information-theoretic lower bound}

Our lower bound shows that $\Omega(n/\eps)$ samples are necessary, \emph{even when a known skeleton is provided}.
To show this, we first show that $\Omega(1/\eps)$ samples are required for the case where $n = 3$ by reducing the problem finding an $\eps$-close graph orientation to the problem of \emph{testing} whether samples are drawn from two given distributions.
To accomplish this, we designed a pair of distributions $P_1$ and $P_2$ and a pair of graphs $G_1$ and $G_2$ such that
(1) $P_1$ and $P_2$ have ``small'' squared Hellinger distance;
(2) $P_i$ has zero KL divergence if projected onto $G_i$;
(3) $P_i$ has ``large'' KL divergence if projected onto $G_j$ (for $j \neq i$).
Since the distributions have small squared Hellinger distance, say less than $\eps$, one needs $\Omega(1/\eps)$ samples to distinguish them, thus showing that $\Omega(1/\eps)$ samples are required for the case where $n = 3$.
To obtain a dependency on $n$, we construct $n/3$ independent copies of the above gadget, \`{a} la proof strategy of \citet[Theorem 7.6]{DBLP:journals/corr/abs-2011-04144}.

\subsection{Other related work}

Structure learning of Bayesian networks is an old problem in machine learning and statistics that has been intensively studied, e.g.\ see \citet[Chapter 18]{DBLP:books/daglib/0023091}.
Many early approaches required faithfulness, a condition which permits learning of the Markov equivalence class, e.g.\ \cite{spirtes1991algorithm, chickering2002optimal, friedman2013learning}.
Finite sample complexity of such algorithms assuming faithfulness-like conditions has also been studied, e.g.\ \cite{friedman1996sample}.
An alternate line of more modern work has considered various other distributional assumptions that permits for efficient learning, e.g.\ \cite{chickering2002finding, hoyer2008nonlinear, shimizu2006linear, peters2014identifiability,ghoshal2017learning,park2017learning,aragam2019globally}, with the final three also showing finite sample complexities.
Specifically for polytrees,
\cite{DBLP:journals/corr/abs-1304-2736} and \cite{DBLP:conf/aaai/GeigerPP90} studied recovery of the DAG for polytrees under the infinite sample regime.

\cite{DBLP:journals/jmlr/AbbeelKN06} studied the problem of efficiently learning a bounded degree factor graph.
Using their method and conversion scheme between factor graphs and Bayes nets, one could efficiently learn polytrees (Bayes nets) with bounded in- \emph{and} out-degrees.
However, as we only consider an upper bound on the in-degrees in this work, directly applying their method scales badly in sample complexity (exponential in the number of variables) for even the simple star-like\footnote{A center node $v$ with undirected edges to the rest of the $n - 1$ nodes; v's in-degree is $d$ and out-degree is $n-d-1$.} polytree.

More recently, \cite{gao2021efficient} studied the more general problem of learning Bayes nets, and their sufficient conditions simplified in the setting of polytrees.
Their approach emphasizes exact recovery, and thus the sample complexity has to depend on the minimum gap of some key mutual information terms.
In contrast, we allow the algorithm to make mistakes when certain mutual information terms are too small to detect for the given sample complexity budget and achieve a PAC-type guarantee.
As such, once the underlying skeleton is discovered, our sample complexity only depends on the $d, n, \eps$ and not on any distributional parameters.

There are also existing works on Bayes net learning with tight bounds in total variation distance with a focus on sample complexity (and not necessarily computational efficiency), e.g.\ \cite{CanonneDKS20}.
Meanwhile, \cite{AcharyaBDK18} consider the problem of learning (in TV distance) a bounded-degree causal Bayes net from interventions, assuming the underlying DAG is known.

\subsection{Outline of paper}

We introduce preliminary notions and related work in \cref{sec:prelims}.
\cref{sec:known-skeleton-and-max-in-degree} then shows how to recover a polytree close in KL divergence, assuming knowledge of the skeleton and maximum in-degree.
\cref{sec:skeleton-assumption} gives sufficient conditions to recover the underlying skeleton from samples, while \cref{sec:lower-bound} provides our sample complexity lower bound (which holds \emph{even when the skeleton is given}).
We conclude in \cref{sec:conclusion} with some open directions and defer some proofs to the appendix.

\section{Preliminaries and tools from previous work}
\label{sec:prelims}

\subsection{Preliminary notions and notation}

We write $[k]$ to mean $\{1, \ldots, k\}$ and the disjoint union as $\dot\cup$.
For any set $A$, let $|A|$ denote its size.
We use hats to denote estimated quantities, e.g., $\hat{I}(X;Y)$ will be the estimated mutual information of $I(X;Y)$. We employ the standard asymptotic notation $O(\cdot)$, $\Omega(\cdot)$ $\Theta(\cdot)$, and write $\tilde{O}(\cdot)$ to omit polylogarithmic factors. Throughout, we identify probability distributions over discrete sets with their probability mass functions (pmf). We use $d^*$ to denote the true maximum in-degree of the underlying polytree, and $d$ to denote an upper bound of $d^*$.
\smallskip

\subsection{Probability distribution definitions}
\label{ssec:prob_dist_def}

\begin{definition}[KL divergence and squared Hellinger distance]
For distributions $P, Q$ defined on the same discrete support $\mathcal{X}$, their KL divergence and squared Hellinger distances are defined as $\kl(P \;\|\; Q) = \sum_{x \in \cX} P(x) \log \frac{P(x)}{Q(x)}$ and $\sqhell(P,Q) = 1 - \sum_{x \in \cX} \sqrt{P(x) \cdot Q(x)}$ respectively.
\end{definition}

Abusing notation, for a distribution $\pd$ on variables $X = \{X_1, \ldots, X_n\}$, we write $\pd_S$ to mean the projection of $\pd$ to the subset of variables $S \subseteq X$ and $\pd_G$ to mean the projection of $\pd$ onto a graph $G$.
More specifically, we have $\pd_G(x_1, \ldots, x_n) = \prod_{x \in X} \pd(x \mid \pi_G(x))$ where $\pi_G(x)$ are the parents of $x$ in $G$.
Note that $P_G$ is the closest distribution\footnote{One can verify this using \cite[Lemma 3.3]{DBLP:journals/corr/abs-2011-04144}: For any distribution $Q$ defined on graph $G$, we have $\kl(P \;\|\; Q) - \kl(P \;\|\; P_G) = \sum_{v \in V} P(\pi_G(v)) \cdot \kl(P(v \mid \pi_G(v)) \;\|\; Q(v \mid \pi_G(v))) \geq 0$.} on $G$ to $P$ in $\kl$, i.e.\ $P_G = \operatorname{argmin}_{Q \in G} \kl(P \;\|\; Q)$.
By \cite{DBLP:journals/tit/ChowL68}, we also know that
\begin{equation}
\label{eq:mutual_information_decom_chou_liu}
\kl(\pd, \pd_G)
= -\sum_{i=1}^n I(X_i; \pi_G(X_i)) - H(\pd_X) + \sum_{i=1}^n H(\pd_{X_i}) \;, 
\end{equation}
where $H$ is the entropy function.
Note that only the first term depends on the graph structure of $G$.

Our goal is to obtain approximately good graph $\hat{G}$ of $P$ in the sense of $\kl(\pd \;\|\; \pd_{\hat{G}}) \leq \eps$.
With $\hat{G}$, one can employ sample and computational-efficient parameter learning algorithms to output the final hypothesis $\hat{\pd}$.
While one can always model the graph $G$ as a clique (a fully connected DAG) in order to satisfy the $\varepsilon$-close requirement, it takes $\exp(n)$ number of samples in the parameter learning step.
Hence, we are interested in obtaining a graph $G$ that is both $\varepsilon$-close and facilitates efficient learning algorithms.
For instance, if $\pd$ was defined on a Bayes net with max in-degree $d$, then we want $G$ to also be a Bayes net with max in-degree $d$.
This is always possible in $\exp(n)$ time by formulating the search of $G$ as an optimization problem that maximizes the summation of mutual information term (first term in \eqref{eq:mutual_information_decom_chou_liu}); see \citep{hoffgen1993learning} and why it is NP-hard in general.
As each mutual information term can be well-estimated, an $\eps$-close graph could be obtained by optimizing over the empirical mutual information scores.
In the case of tree-structured distributions, this can be efficiently solved by using any maximum spanning tree algorithm.
If any valid topological ordering of the target Bayes net $P$ is present, then an efficient greedy approach is able to solve the problem.

\begin{definition}[(Conditional) Mutual Information]
Given a distribution $\pd$, the mutual information (MI) of two random variables $X$ and $Y$, supported on $\cX$ and $\cY$ respectively, is defined as
\(
I(X;Y)
= \sum_{x \in \cX, y \in \cY} \pd(x,y) \cdot \log\left(\frac{\pd(x,y)}{\pd(x) \cdot \pd(y)}\right).
\)
Conditioning on a third random variable $Z$, supported on $\cZ$, the conditional MI is defined as:
\(
I(X;Y \mid Z)
= \sum_{x \in \cX, y \in \cY, z \in \cZ} \pd(x,y,z) \cdot \log\left(\frac{\pd(x,y,z) \cdot \pd(z)}{\pd(x,z) \cdot \pd(y,z)}\right).
\)
\end{definition}

We will employ the following result (\cref{cor:CMI_tester}) about conditional mutual information testers, adapted from Theorem 1.3 of \cite{DBLP:journals/corr/abs-2011-04144}; see \cref{sec:adapting-known-tester-result} for derivation details and \cref{sec:appendix-explicit-constant} for a derivation of a constant $C$ that works.

\begin{restatable}[Conditional MI tester]{mycorollary}{CMItester}
\label{cor:CMI_tester}
Fix any $\eps > 0$.
Let $(X, Y, Z)$ be three random variables over $\Sigma_X, \Sigma_Y, \Sigma_Z$ respectively.
Given the empirical distribution $(\hat{X}, \hat{Y}, \hat{Z})$ over a size $N$ sample of $(X, Y, Z)$, there exists a universal constant $0 < C < 1$ so that for any
\[
N
\geq \Theta \left( \frac{|\Sigma_X| \cdot |\Sigma_Y| \cdot |\Sigma_Z|}{\eps} \cdot \log \frac{|\Sigma_X| \cdot |\Sigma_Y| \cdot |\Sigma_Z|}{\delta} \cdot \log \frac{|\Sigma_X| \cdot |\Sigma_Y| \cdot |\Sigma_Z| \cdot \log(1 / \delta)}{\eps} \right),
\]
the following statements hold with probability $1 - \delta$:\\
(1) If $I(X; Y \mid Z) = 0$, then $\hat{I}(X; Y \mid Z) < C \cdot \eps$.\\
(2) If $\hat{I}(X; Y \mid Z) \leq C \cdot \eps$, then $I(X; Y \mid Z) < \eps$.\\
Unconditional statements involving $I(X; Y)$ and $\hat{I}(X; Y)$ hold similarly by setting $|\Sigma_Z| = 1$.
\end{restatable}

Using the contrapositive of the first statement of \cref{cor:CMI_tester} and non-negativity of conditional mutual information, one can also see that if $\hat{I}(X; Y \mid Z) \geq C \cdot \eps$, then $I(X; Y \mid Z) > 0$.
 
\subsection{Graphical notions}

Let $G = (V,E)$ be a graph on $|V| = n$ vertices and $|E|$ edges where adjacencies are denoted with dashes, e.g.\ $u - v$.
For any vertex $v \in V$, we use $N(v) \subseteq V \setminus \{v\}$ to denote the neighbors of $v$ and $d(v) = |N(v)|$ to denote $v$'s degree.
An undirected cycle is a sequence of $k \geq 3$ vertices such that $v_1 - v_2 - \ldots - v_k - v_1$.
For any subset $E' \subseteq E$ of edges, we say that the graph $H = (V,E')$ is the edge-induced subgraph of $G$ with respect to $E'$.

We use arrows to denote directed edges, e.g.\ $u \to v$, use $\pi(v)$ to denote the parents of $v$, and use $d^{in}(v)$ to denote $v$'s incoming degree.
An interesting directed subgraph on three vertices is the \emph{v-structure}: $u \to v \gets w$, $u \notbackslash w$, and $v$ is called the center of the v-structure.

In this work, we study a generalized version of v-structures (\emph{deg-$\ell$ v-structure}) where the center has $\ell \geq 2$ parents $u_1, u_2, \ldots, u_{\ell}$.
We say that a deg-$\ell$ v-structure is said to be $\eps$-strong if we can reliably identify them in the finite sample regime.

\begin{definition}[$\varepsilon$-strong deg-$\ell$ v-structure]
Let $0 < C < 1$ be the universal constant appearing in \cref{cor:CMI_tester}.
A deg-$\ell$ v-structure is a subgraph on $\ell+1$ nodes $v, u_1, \ldots, u_\ell$ such that:\\
1. \textbf{deg-$\ell$ v-structure}: $v \gets u_k$ for all $k \in [\ell]$, and $u_k \notbackslash u_{k'}$ for all $k, k' \in [\ell]$ and $k \neq k'$\\
2. \textbf{$\eps$-strong}: $I(u_k ; \{ u_1, u_2, \ldots, u_{\ell} \} \setminus u_k \mid v) \geq C \cdot \eps$ for all $k \in [\ell]$
\end{definition}
Meek rules \citep{meek1995} are a set of 4 edge orientation rules that are sound and complete with respect to any partially oriented graph that can be fully oriented in an acyclic fashion.
Soundness means that any orientation due to rule invocations are correct while completeness means that any orientation that could theoretically be recovered would be recovered.
So, one can always repeatedly apply Meek rules till a fixed point to maximize the number of oriented arcs.
One particular orientation rule (Meek $R1$)\footnote{The other 3 rules involve undirected cycles in the graph and are thus irrelevant in our context where the underlying undirected graph is a tree.} orients $b \to c$ whenever a partially oriented graph has the configuration $a \to b - c$ and $a \notbackslash c$ so as to avoid forming a new v-structure of the form $a \to b \gets c$.
In the same spirit, we define Meek $R1(d)$ to orient all incident unoriented edges away from $v$ whenever $v$ already has $d$ parents in a partially oriented graph.

The \emph{skeleton} $\skel(G)$ of a graph $G$ refers to the resulting undirected graph after unorienting all edges in $G$, e.g.\ see \cref{fig:running-example}.
A graph $G$ is a \emph{polytree} if $\skel(G)$ is a forest.
For $d \geq 1$, a polytree $G$ is a $d$-polytree if all vertices in $G$ have at most $d$ parents.
When we \emph{freely orient} a forest, we pick arbitrary root nodes in the connected components and orient to form a 1-polytree.
In polytrees, there is a unique path in $skel(G)$ between any two vertices in the graph.
As such, the ancestors of a vertex $v$ are mutually independent, and typically become mutually dependent when $v$ (or any of $v$'s descendants) are being conditioned over.
\section{Recovering a good orientation given a skeleton and degree bound}
\label{sec:known-skeleton-and-max-in-degree}

In this section, we describe and analyze an algorithm for estimating a probability distribution $\pd$ that is defined on a $d^*$-polytree $G^*$.
We assume that we are given $\skel(G^*)$ and $d$ as input, where $d^* \leq d$.

Note that for some distributions there could be more than one ground truth graph, e.g.\ when the Markov equivalence class has multiple graphs.
In such situations, for analysis purposes, we are free to choose any graph that $\pd$ is Markov with respect to.
As the mutual information scores\footnote{The mutual information score is the sum of the mutual information terms as in \cref{eq:mutual_information_decom_chou_liu}.} are the same for any graphs that $\pd$ is Markov with respect to, the choice of $G^*$ does not matter here.

\subsection{Algorithm}
\label{sec:algo}

At any point in the algorithm, let us define the following sets.
Let $N(v)$ be the set of all neighbors of $v$ in $\skel(G^*)$.
Let $N^{\mathrm{in}}(v) \subseteq N(v)$ be the current set of incoming neighbors of $v$.
Let $N^{\mathrm{out}}(v) \subseteq N(v)$ be the current set of outgoing neighbors of $v$.
Let $N^{\mathrm{un}}(v) \subseteq N(v)$ be the current set of unoriented neighbors of $v$.
That is,
$
N(v) = N^{\mathrm{in}}(v) \;\dot\cup\; N^{\mathrm{out}}(v) \;\dot\cup\; N^{\mathrm{un}}(v)
$.

\begin{algorithm}[tb]
\SetAlgoNoEnd
\setcounter{AlgoLine}{0}
\caption{Algorithm for known skeleton and max in-degree.}
\label{alg:known-skeleton-and-max-in-degree}
\KwData{Skeleton $\skel(G^*) = (V,E)$, max in-degree $d$, threshold $\eps > 0$, universal constant $C$}
\KwResult{A complete orientation of $\skel(G^*)$}

Run Phase 1: Orient strong v-structures (\cref{alg:phase1}) \Comment*[f]{$\cO(n^{d+1})$ time}

Run Phase 2: Local search and Meek $R1(d)$ (\cref{alg:phase2}) \Comment*[f]{$\cO(n^3)$ time}

Run Phase 3: Freely orient remaining unoriented edges (\cref{alg:phase3}) \Comment*[f]{$\cO(n)$ time via DFS}

\KwRet{$\hat{G}$}
\end{algorithm}

Our algorithm has three phases.
In Phase 1, we orient strong v-structures.
In Phase 2, we locally check if an edge is forced to orient one way or another to avoid incurring too much error.\footnote{Note that within the for loop from Line 7 of \cref{alg:phase2}, neither condition may hold, in which case we do not orient anything, hence the ``missing'' else.}
In Phase 3, we orient the remaining unoriented edges as a 1-polytree.
Since the remaining edges were not forced, we may orient the remaining edges in an arbitrary direction (while not incurring ``too much error'') as long as the final incoming degrees of any vertex does not increase by more than 1.
Subroutine \textsc{Orient} (\cref{alg:orient}) performs the necessary updates when we orient $u - v$ to $u \to v$.

\begin{algorithm}[tb]
\SetAlgoNoEnd
\setcounter{AlgoLine}{0}
\caption{\textsc{Orient}: Subroutine to orient edges}
\label{alg:orient}
\KwData{Vertices $u$ and $v$ where $u - v$ is currently unoriented}

Orient $u - v$ as $u \to v$.

Update $N^{\mathrm{in}}(v)$ to $N^{\mathrm{in}}(v) \cup \{u\}$ and $N^{\mathrm{un}}(v)$ to $N^{\mathrm{un}}(v) \setminus \{u\}$.

Update $N^{\mathrm{out}}(u)$ to $N^{\mathrm{out}}(u) \cup \{v\}$ and $N^{\mathrm{un}}(u)$ to $N^{\mathrm{un}}(u) \setminus \{v\}$.
\end{algorithm}

\begin{algorithm}[tb]
\SetAlgoNoEnd
\setcounter{AlgoLine}{0}
\caption{Phase 1: Orient strong v-structures}
\label{alg:phase1}
\KwData{Skeleton $\skel(G^*) = (V,E)$, max in-degree $d$, threshold $\eps > 0$, universal constant $C$}

$\gamma \gets d$

\While{$\gamma \geq 2$}
{
    \For(\Comment*[f]{Arbitrary order}){$v \in V$}
    {

        \For(\Comment*[f]{$\cN_{\gamma} \subseteq 2^{N(v)}$ are the $\gamma$ neighbors of $v$; $|\cN_{\gamma}| = \binom{|N(v)|}{\gamma}$}){$T \in \cN_{\gamma}$}
        {
            \If{\label{alg:line:orient_condition}$|T \cup N^{\mathrm{in}}(v)| \leq d$ \textbf{and} $\hat{I}(u; T \setminus \{u\} \mid v) \geq C \cdot \eps$, $\forall u \in T$}
            {
                \For(\Comment*[f]{Strong deg-$\gamma$ v-structure}){$u \in T$} {
                    \textsc{Orient}($u$, $v$)
                }
            }
        }
    }
    $\gamma \gets \gamma-1$ \Comment*[f]{Decrement degree bound}
}
\end{algorithm}

\textbf{Example}$\quad$
Suppose we have the partially oriented graph \cref{fig:running-example}(c) after Phase 1.
Since $N^{\mathrm{in}}(d) = \{a,b\}$, we will check the edge orientations of $c - d$ and $f - d$.
Since $I(f; \{a,b\} \mid d) = 0$, we will have $\hat{I}(f; \{a,b\} \mid d) \leq \eps$, so we will \emph{not} erroneously orient $f \to d$.
Meanwhile, $I(c; \{a,b\}) = 0$, we will have $\hat{I}(c; \{a,b\}) \leq \eps$, so we will \emph{not} erroneously orient $d \to c$.

\begin{algorithm}[tb]
\SetAlgoNoEnd
\setcounter{AlgoLine}{0}
\caption{Phase 2: Local search and Meek $R1(d)$}
\label{alg:phase2}
\KwData{Skeleton $\skel(G^*) = (V,E)$, max in-degree $d$, threshold $\eps > 0$, universal constant $C$}

\Do(\Comment*[f]{$\cO(n)$ iterations, $\cO(n^2)$ time per iteration}){new edges are being oriented} {
    \If(\Comment*[f]{Meek $R1(d)$}){$\exists v \in V$ such that $|N^{\mathrm{in}}(v)| = d$ and $N^{\mathrm{un}}(v) \neq \emptyset$} {
        Orient all unoriented arcs \emph{away} from $v$

        Update $N^{\mathrm{out}}(v) \gets N^{\mathrm{out}}(v) \cup N^{\mathrm{un}}(v)$; $N^{\mathrm{un}}(v) \gets \emptyset$
    }
    
    \For{every node $v \in V$} {
        \If{$1 \leq |N^{\mathrm{in}}(v)| < d$} {
            \For{every $u \in N^{\mathrm{un}}(v)$} {
                \lIf{$\hat{I}(u; N^{\mathrm{in}}(v) \mid v) > C \cdot \eps$} {
                    \textsc{Orient}($u$, $v$)
                }
                \lElseIf{$\hat{I}(u; N^{\mathrm{in}}(v)) > C \cdot \eps$} {
                    \textsc{Orient}($v$, $u$)
                }
            }
        }
    }
}
\end{algorithm}

\begin{algorithm}[tb]
\SetAlgoNoEnd
\setcounter{AlgoLine}{0}
\caption{Phase 3: Freely orient remaining unoriented edges}
\label{alg:phase3}
\KwData{Skeleton $\skel(G^*) = (V,E)$, max in-degree $d$, threshold $\eps > 0$, universal constant $C$}

Let $H$ be the forest induced by the remaining unoriented edges.

Freely orient $H$ as a 1-polytree (i.e.\ maximum in-degree in $H$ is 1).

Let $\hat{G}$ be the combination of the oriented $H$ and the previously oriented arcs.

\KwRet{$\hat{G}$}
\end{algorithm}

\subsection{Analysis}

In our subsequent analysis, we rely on the conclusions of \cref{cor:CMI_tester} with error tolerance $\eps' = \frac{\eps}{2 n \cdot (d + 1)}$.
Via a union bound over $\cO(n^{d + 1})$ events, \cref{lem:whp} ensures that \emph{all} our (conditional) MI tests in \cref{alg:known-skeleton-and-max-in-degree} will behave as expected with probability at least $1-\delta$, with sufficient samples.

\begin{mylemma}
\label{lem:whp}
Suppose all variables in the Bayesian network have alphabet $\Sigma$, for $|\Sigma| \geq 2$.
For $\eps' > 0$, $\cO(n^{d+1})$ statements of the following forms all simultaneously succeed with probability at least $1 - \delta$:\\
\hspace{1em}(1) If $I(\bX; \bY \mid Z) = 0$, then $\hat{I}(\bX; \bY \mid Z) < C \cdot \eps'$,\\
\hspace{1em}(2) If $\hat{I}(\bX; \bY \mid Z) \leq C \cdot \eps'$, then $I(\bX; \bY \mid Z) < \eps'$.\\
with $m$ empirical samples, where $Z \in V \cup \{\emptyset\}$, $\bX, \bY \subseteq V \setminus \{Z\}$, $|\bX \;\dot\cup\; \bY| \leq d$, and
\[
m
\in \cO \left( \frac{|\Sigma|^{d + 1}}{\eps'} \cdot \log \frac{|\Sigma|^{d + 1} \cdot n^{d}}{\delta} \cdot \log \frac{|\Sigma|^{d + 1} \cdot \log(n^{d} / \delta)}{\eps'} \right)
\]
\end{mylemma}
\begin{proof}
Set $\eps'=\frac{\eps}{2 n \cdot (d + 1)}$, then use \cref{cor:CMI_tester} and apply union bound over $\cO(n^{d + 1})$ tests.
\end{proof}

In the remaining of our analysis, we will analyze under the assumption that all our $\cO(n^{d+1})$ tests are correct with the required tolerance level.
Full proofs are deferred to \cref{sec:appx-algo-analysis}.
\medskip

Recall that $\pi(v)$ is the set of true parents of $v$ in $G^*$.
Let $H$ be the forest induced by the remaining unoriented edges after Phase 2 and $\hat{G}$ be returned graph of \cref{alg:known-skeleton-and-max-in-degree}.
Let us denote the final $N^{\mathrm{in}}(v)$ as $\pi^{\mathrm{in}}(v)$ at the end of Phase 2, just before freely orienting, i.e.\ the vertices pointing into $v$ in $\hat{G} \setminus H$.
Then, $\pi^{\mathrm{un}}(v) = \pi(v) \setminus \pi^{\mathrm{in}}(v)$ is the set of ground truth parents that are not identified in both Phase 1 and Phase 2.
\cref{lem:oriented-arcs-in-Phase1-are-ground-truth-orientations} argues that  the algorithm does not make mistakes for orientations in $\hat{G} \setminus H$, so all edges in $\pi^{\mathrm{un}}(v)$ will be unoriented at the end of Phase 2.

\begin{restatable}{mylemma}{orientedarcsinPhaseonearegroundtruthorientations}
\label{lem:oriented-arcs-in-Phase1-are-ground-truth-orientations}
Any oriented arc in $\hat{G} \setminus H$ is a ground truth orientation.
That is, any vertex parent set in $\hat{G} \setminus H$ is a subset of $\pi(v)$, i.e.\ $\pi^{\mathrm{in}}(v) \subseteq \pi(v)$, and $N^{\mathrm{in}}(v)$ at any time during the algorithm will have $N^{\mathrm{in}}(v) \subseteq \pi^{\mathrm{in}}(v)$.
\end{restatable}

Let $\hat{\pi}(v)$ be the proposed parents of $v$ output by \cref{alg:known-skeleton-and-max-in-degree}.
The KL divergence between the true distribution and our output distribution is $\sum_{v \in V} I(v; \pi(v)) - \sum_{v \in V} I(v; \hat{\pi}(v))$ as the structure independent terms will cancel out.
To get a bound on the KL divergence, we will upper bound $\sum_{v \in V} I(v; \pi(v))$ and lower bound $\sum_{v \in V} I(v; \hat{\pi}(v))$.
To upper bound $I(v; \pi(v))$ in terms of $\pi^{\mathrm{in}}(v) \subseteq \pi(v)$ and $I(v;u)$ for $u \in \pi^{\mathrm{un}}(v)$, we use \cref{lem:decomposition_sequence_exists_for_missing_v_structure} which relies on repeated applications of \cref{lem:decomposition_is_possible_when_v_structure_is_absent}.
To lower bound $\sum_{v \in V} I(v; \hat{\pi}(v))$, we use \cref{lemma:single-extra-incoming-is-okay}.

\begin{restatable}{mylemma}{decompositionispossiblewhenvstructureisabsent}
\label{lem:decomposition_is_possible_when_v_structure_is_absent}
Fix any vertex $v$, any $S \subseteq \pi^{\mathrm{un}}(v)$, and any $S' \subseteq \pi^{\mathrm{in}}(v)$.
If $S \neq \emptyset$, then there exists a vertex $u \in S \cup S'$ with
\begin{equation}
\label{eq:decomposition_exists_if_v_structure_absent}
I(v ; S \cup S') \leq I(v ; S \cup S' \setminus \{u\}) + I(v ; u) + \eps \;.
\end{equation}
\end{restatable}

\begin{restatable}{mylemma}{decompositionsequenceexistsformissingvstructure}
\label{lem:decomposition_sequence_exists_for_missing_v_structure}
For any vertex $v$ with $\pi^{\mathrm{in}}(v)$, we can show that
\[
I(v ; \pi(v)) \leq \eps \cdot |\pi(v)| + I (v ; \pi^{\mathrm{in}}(v)) + \sum_{u \in \pi^{\mathrm{un}}(v)} I (v ; u)\;.
\]
\end{restatable}

In Phase 3, we increase the incoming edges to any vertex by at most one.
The following lemma tells us that we lose at most\footnote{Orienting ``freely'' could also increase the mutual information score and this is considering the worst case.} an additive $\eps$ error per vertex.

\begin{restatable}{mylemma}{singleextraincomingisokay}
\label{lemma:single-extra-incoming-is-okay}
Consider an arbitrary vertex $v$ with $\pi^{\mathrm{in}}(v)$ at the start of Phase 3.
If Phase 3 orients $u \to v$ for some $u - v \in H$, then
\[
I(v; \pi^{\mathrm{in}}(v) \cup \{u\}) \geq I(v; \pi^{\mathrm{in}}(v)) + I(v; u) - \eps.
\]
\end{restatable}

By using \cref{lem:decomposition_sequence_exists_for_missing_v_structure} and \cref{lemma:single-extra-incoming-is-okay}, we can show our desired KL divergence bound (\cref{lem:final-output-is-good}).

\begin{restatable}{mylemma}{finaloutputisgood}
\label{lem:final-output-is-good}
Let $\pi(v)$ be the true parents of $v$.
Let $\hat{\pi}(v)$ be the proposed parents of $v$ output by our algorithm.
Then,
\[
\sum_{v \in V} I(v; \pi(v)) - \sum_{v \in V} I(v; \hat{\pi}(v))
\leq n \cdot (d^* + 1) \cdot \eps \;.
\]
\end{restatable}

Note that we have a bound with respect to the true max-degree $d^*$ despite only given an upper bound $d$ as input.
With these results in hand, we are ready to establish our main theorem.

\begin{proof}%
\textbf{of~\cref{theo:main}}\quad
We first combine \cref{lem:final-output-is-good} and \cref{lem:whp} with $\eps' = \frac{\eps}{2 n \cdot (d + 1)} \leq \frac{\eps}{2 n \cdot (d^* + 1)}$ in order to obtain an orientation $\hat{G}$ which is close to $G^*$.
Now, much similar to the proof of \citet[Theorem 1.4]{DBLP:journals/corr/abs-2011-04144}, we recall that there exist efficient algorithms for estimating the parameters of a Bayes net with in-degree-$d$ (note that this includes $d$-polytrees) $\pd$ once a close-enough graph $\hat{G}$ is recovered \citep{DBLP:journals/ml/Dasgupta97, BhattacharyyaGMV20}, with sample complexity $\tilde{\cO}(n \cdot |\Sigma|^d / \eps)$.
Denote the final output $\hat{\pd}_{\hat{G}}$, a distribution that is estimated using the conditional probabilities implied by $\hat{G}$.
One can bound the KL divergences as follows:
\[
\kl(\pd \;\|\; P_{\hat{G}}) - \kl(\pd \;\|\; P_{G^*})
\leq \eps / 2
\quad \text{and} \quad
\kl(\pd \;\|\; \hat{P}_{\hat{G}}) - \kl(\pd \;\|\; P_{\hat{G}}) \leq \eps / 2 \;.
\]
The first inequality follows from our graph learning guarantees on $\hat{G}$ while the second is due to performing parameter learning algorithms on $\hat{G}$.
Thus,
$
\kl(\pd \;\|\; \hat{P}_{\hat{G}})
\leq \eps + \kl(\pd \;\|\; P_{G^*})
= \eps
$.
\end{proof}

\begin{figure}[htb]
\centering
\begin{subfigure}[Ground truth $G^*$. $\pi(d) = \{a,b,c\}$]{
    \centering
    \resizebox{0.22\linewidth}{!}{
    \resizebox{0.7\linewidth}{!}{%
\begin{tikzpicture}
    \node[draw, circle, minimum size=15pt, inner sep=2pt] at (0,0) (d) {$d$};
    \node[draw, circle, minimum size=15pt, inner sep=2pt, left=of d] (b) {$b$};
    \node[draw, circle, minimum size=15pt, inner sep=2pt, above=of b] (a) {$a$};
    \node[draw, circle, minimum size=15pt, inner sep=2pt, below=of b] (c) {$c$};
    \node[draw, circle, minimum size=15pt, inner sep=2pt, right=of d] (f) {$f$};
    \node[draw, circle, minimum size=15pt, inner sep=2pt, above=of f] (e) {$e$};
    \node[draw, circle, minimum size=15pt, inner sep=2pt, below=of f] (g) {$g$};
    \node[draw, circle, minimum size=15pt, inner sep=2pt, right=of e] (h) {$h$};
    \node[draw, circle, minimum size=15pt, inner sep=2pt, right=of f] (i) {$i$};
    \node[draw, circle, minimum size=15pt, inner sep=2pt, right=of g] (j) {$j$};

    \draw[thick, -{Stealth[scale=1.5]}] (a) -- (d);
    \draw[thick, -{Stealth[scale=1.5]}] (b) -- (d);
    \draw[thick, -{Stealth[scale=1.5]}] (c) -- (d);
    \draw[thick, -{Stealth[scale=1.5]}] (d) -- (f);
    \draw[thick, -{Stealth[scale=1.5]}] (e) -- (f);
    \draw[thick, -{Stealth[scale=1.5]}] (f) -- (g);
    \draw[thick, -{Stealth[scale=1.5]}] (e) -- (h);
    \draw[thick, -{Stealth[scale=1.5]}] (g) -- (i);
    \draw[thick, -{Stealth[scale=1.5]}] (j) -- (g);
\end{tikzpicture}
}
    }
}
\end{subfigure}
\hfill
\begin{subfigure}[Midway of Phase 1. $N^{in}(d) = \{a,b\}$]{
    \centering
    \resizebox{0.22\linewidth}{!}{
    \resizebox{0.3\linewidth}{!}{%
\begin{tikzpicture}
    \node[draw, circle, minimum size=15pt, inner sep=2pt] at (0,0) (d) {$d$};
    \node[draw, circle, minimum size=15pt, inner sep=2pt, left=of d] (b) {$b$};
    \node[draw, circle, minimum size=15pt, inner sep=2pt, above=of b] (a) {$a$};
    \node[draw, circle, minimum size=15pt, inner sep=2pt, below=of b] (c) {$c$};
    \node[draw, circle, minimum size=15pt, inner sep=2pt, right=of d] (f) {$f$};
    \node[draw, circle, minimum size=15pt, inner sep=2pt, above=of f] (e) {$e$};
    \node[draw, circle, minimum size=15pt, inner sep=2pt, below=of f] (g) {$g$};
    \node[draw, circle, minimum size=15pt, inner sep=2pt, right=of e] (h) {$h$};
    \node[draw, circle, minimum size=15pt, inner sep=2pt, right=of f] (i) {$i$};
    \node[draw, circle, minimum size=15pt, inner sep=2pt, right=of g] (j) {$j$};

    \draw[thick, -{Stealth[scale=1.5]}] (a) -- (d);
    \draw[thick, -{Stealth[scale=1.5]}] (b) -- (d);
    \draw[thick] (c) -- (d);
    \draw[thick] (d) -- (f);
    \draw[thick] (e) -- (f);
    \draw[thick] (f) -- (g);
    \draw[thick] (e) -- (h);
    \draw[thick] (g) -- (i);
    \draw[thick] (j) -- (g);
\end{tikzpicture}
}
    }
}
\end{subfigure}
\hfill
\begin{subfigure}[Before final phase. $\pi^{in}(d) = \{a,b\}$]{
    \centering
    \resizebox{0.22\linewidth}{!}{
    \resizebox{0.7\linewidth}{!}{%
\begin{tikzpicture}
    \node[draw, circle, minimum size=15pt, inner sep=2pt] at (0,0) (d) {$d$};
    \node[draw, circle, minimum size=15pt, inner sep=2pt, left=of d] (b) {$b$};
    \node[draw, circle, minimum size=15pt, inner sep=2pt, above=of b] (a) {$a$};
    \node[draw, circle, minimum size=15pt, inner sep=2pt, below=of b] (c) {$c$};
    \node[draw, circle, minimum size=15pt, inner sep=2pt, right=of d] (f) {$f$};
    \node[draw, circle, minimum size=15pt, inner sep=2pt, above=of f] (e) {$e$};
    \node[draw, circle, minimum size=15pt, inner sep=2pt, below=of f] (g) {$g$};
    \node[draw, circle, minimum size=15pt, inner sep=2pt, right=of e] (h) {$h$};
    \node[draw, circle, minimum size=15pt, inner sep=2pt, right=of f] (i) {$i$};
    \node[draw, circle, minimum size=15pt, inner sep=2pt, right=of g] (j) {$j$};

    \draw[thick, -{Stealth[scale=1.5]}] (a) -- (d);
    \draw[thick, -{Stealth[scale=1.5]}] (b) -- (d);
    \draw[thick, red] (c) -- (d);
    \draw[thick, red] (d) -- node[above=10pt,midway] {\Large $H$} (f);
    \draw[thick, red] (e) -- (f);
    \draw[thick, -{Stealth[scale=1.5]}] (f) -- (g);
    \draw[thick, red] (e) -- (h);
    \draw[thick, -{Stealth[scale=1.5]}] (g) -- (i);
    \draw[thick, -{Stealth[scale=1.5]}] (j) -- (g);
\end{tikzpicture}
}
    }
}
\end{subfigure}
\hfill
\begin{subfigure}[Proposed graph $\hat{G}$. $\hat{\pi}(d) = \{a,b,f\}$]{
    \centering
    \resizebox{0.22\linewidth}{!}{\resizebox{0.7\linewidth}{!}{%
\begin{tikzpicture}
    \node[draw, circle, minimum size=15pt, inner sep=2pt] at (0,0) (d) {$d$};
    \node[draw, circle, minimum size=15pt, inner sep=2pt, left=of d] (b) {$b$};
    \node[draw, circle, minimum size=15pt, inner sep=2pt, above=of b] (a) {$a$};
    \node[draw, circle, minimum size=15pt, inner sep=2pt, below=of b] (c) {$c$};
    \node[draw, circle, minimum size=15pt, inner sep=2pt, right=of d] (f) {$f$};
    \node[draw, circle, minimum size=15pt, inner sep=2pt, above=of f] (e) {$e$};
    \node[draw, circle, minimum size=15pt, inner sep=2pt, below=of f] (g) {$g$};
    \node[draw, circle, minimum size=15pt, inner sep=2pt, right=of e] (h) {$h$};
    \node[draw, circle, minimum size=15pt, inner sep=2pt, right=of f] (i) {$i$};
    \node[draw, circle, minimum size=15pt, inner sep=2pt, right=of g] (j) {$j$};

    \draw[thick, -{Stealth[scale=1.5]}, blue] (a) -- (d);
    \draw[thick, -{Stealth[scale=1.5]}, blue] (b) -- (d);
    \draw[thick, -{Stealth[scale=1.5]}, red] (d) -- (c);
    \draw[thick, -{Stealth[scale=1.5]}, red] (f) -- (d);
    \draw[thick, -{Stealth[scale=1.5]}, red] (e) -- (f);
    \draw[thick, -{Stealth[scale=1.5]}, blue] (f) -- (g);
    \draw[thick, -{Stealth[scale=1.5]}, red] (h) -- (e);
    \draw[thick, -{Stealth[scale=1.5]}, blue] (g) -- (i);
    \draw[thick, -{Stealth[scale=1.5]}, blue] (j) -- (g);
\end{tikzpicture}
}}
}
\end{subfigure}
\caption{
An example run to illusrate notations.
In $G^*$, vertex $d$ has parents $\pi(d) = \{a,b,c\}$. While the algorithm executes, we track a tentative parent set $N^{in}(d)$ of $d$ and fix it to $\pi^{in}(d)$ right before the final phase.
Since $d = 3$, observe that $g \to i$ must have been oriented due to a local search step and \emph{not} due to Meek $R1(3)$ in Phase 2.
At the end, in $\hat{G}$, the proposed parent set of $d$ is $\hat{\pi}(d) = \{a,b,f\}$.
Note that $\hat{G}$ only shows one possible orientation of the red unoriented subgraph $H$ before the final phase; see \cref{fig:all-five-possible-H-orientations} for others.
}
\label{fig:lemma6-notation}
\end{figure}
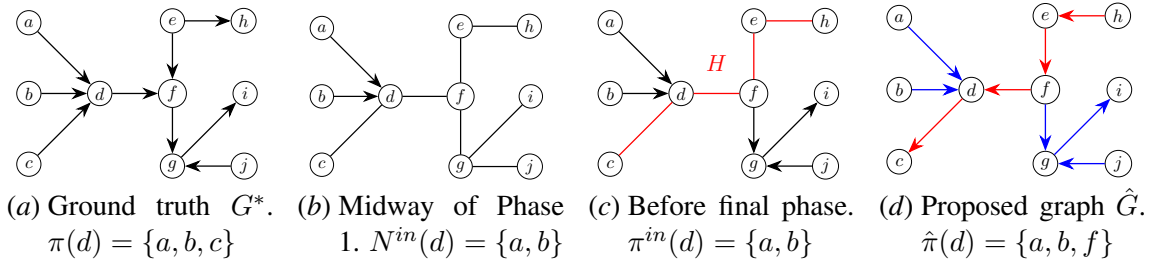

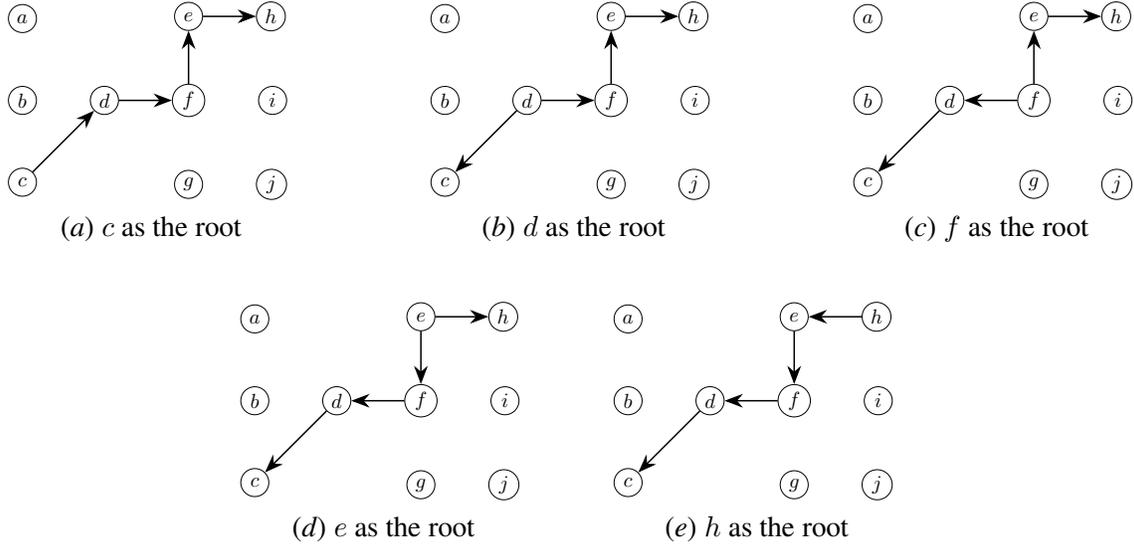
\begin{figure}[htb]
\centering
\begin{subfigure}[$c$ as the root][t]{
    \centering
    \resizebox{0.25\linewidth}{!}{\resizebox{0.7\linewidth}{!}{%
\begin{tikzpicture}
    \node[draw, circle, minimum size=15pt, inner sep=2pt] at (0,0) (d) {$d$};
    \node[draw, circle, minimum size=15pt, inner sep=2pt, left=of d] (b) {$b$};
    \node[draw, circle, minimum size=15pt, inner sep=2pt, above=of b] (a) {$a$};
    \node[draw, circle, minimum size=15pt, inner sep=2pt, below=of b] (c) {$c$};
    \node[draw, circle, minimum size=15pt, inner sep=2pt, right=of d] (f) {$f$};
    \node[draw, circle, minimum size=15pt, inner sep=2pt, above=of f] (e) {$e$};
    \node[draw, circle, minimum size=15pt, inner sep=2pt, below=of f] (g) {$g$};
    \node[draw, circle, minimum size=15pt, inner sep=2pt, right=of e] (h) {$h$};
    \node[draw, circle, minimum size=15pt, inner sep=2pt, right=of f] (i) {$i$};
    \node[draw, circle, minimum size=15pt, inner sep=2pt, right=of g] (j) {$j$};

    \draw[thick, -{Stealth[scale=1.5]}] (c) -- (d);
    \draw[thick, -{Stealth[scale=1.5]}] (d) -- (f);
    \draw[thick, -{Stealth[scale=1.5]}] (f) -- (e);
    \draw[thick, -{Stealth[scale=1.5]}] (e) -- (h);
\end{tikzpicture}
}}
}
\end{subfigure}
\hfill
\begin{subfigure}[$d$ as the root][t]{
    \centering
    \resizebox{0.25\linewidth}{!}{\resizebox{0.7\linewidth}{!}{%
\begin{tikzpicture}
    \node[draw, circle, minimum size=15pt, inner sep=2pt] at (0,0) (d) {$d$};
    \node[draw, circle, minimum size=15pt, inner sep=2pt, left=of d] (b) {$b$};
    \node[draw, circle, minimum size=15pt, inner sep=2pt, above=of b] (a) {$a$};
    \node[draw, circle, minimum size=15pt, inner sep=2pt, below=of b] (c) {$c$};
    \node[draw, circle, minimum size=15pt, inner sep=2pt, right=of d] (f) {$f$};
    \node[draw, circle, minimum size=15pt, inner sep=2pt, above=of f] (e) {$e$};
    \node[draw, circle, minimum size=15pt, inner sep=2pt, below=of f] (g) {$g$};
    \node[draw, circle, minimum size=15pt, inner sep=2pt, right=of e] (h) {$h$};
    \node[draw, circle, minimum size=15pt, inner sep=2pt, right=of f] (i) {$i$};
    \node[draw, circle, minimum size=15pt, inner sep=2pt, right=of g] (j) {$j$};

    \draw[thick, -{Stealth[scale=1.5]}] (d) -- (c);
    \draw[thick, -{Stealth[scale=1.5]}] (d) -- (f);
    \draw[thick, -{Stealth[scale=1.5]}] (f) -- (e);
    \draw[thick, -{Stealth[scale=1.5]}] (e) -- (h);
\end{tikzpicture}
}}
}
\end{subfigure}
\hfill
\begin{subfigure}[$f$ as the root][t]{
    \centering
    \resizebox{0.25\linewidth}{!}{\resizebox{0.7\linewidth}{!}{%
\begin{tikzpicture}
    \node[draw, circle, minimum size=15pt, inner sep=2pt] at (0,0) (d) {$d$};
    \node[draw, circle, minimum size=15pt, inner sep=2pt, left=of d] (b) {$b$};
    \node[draw, circle, minimum size=15pt, inner sep=2pt, above=of b] (a) {$a$};
    \node[draw, circle, minimum size=15pt, inner sep=2pt, below=of b] (c) {$c$};
    \node[draw, circle, minimum size=15pt, inner sep=2pt, right=of d] (f) {$f$};
    \node[draw, circle, minimum size=15pt, inner sep=2pt, above=of f] (e) {$e$};
    \node[draw, circle, minimum size=15pt, inner sep=2pt, below=of f] (g) {$g$};
    \node[draw, circle, minimum size=15pt, inner sep=2pt, right=of e] (h) {$h$};
    \node[draw, circle, minimum size=15pt, inner sep=2pt, right=of f] (i) {$i$};
    \node[draw, circle, minimum size=15pt, inner sep=2pt, right=of g] (j) {$j$};

    \draw[thick, -{Stealth[scale=1.5]}] (d) -- (c);
    \draw[thick, -{Stealth[scale=1.5]}] (f) -- (d);
    \draw[thick, -{Stealth[scale=1.5]}] (f) -- (e);
    \draw[thick, -{Stealth[scale=1.5]}] (e) -- (h);
\end{tikzpicture}
}}
}
\end{subfigure}
\\
\vspace{20pt}
\begin{subfigure}[$e$ as the root][t]{
    \centering
    \resizebox{0.25\linewidth}{!}{\resizebox{0.7\linewidth}{!}{%
\begin{tikzpicture}
    \node[draw, circle, minimum size=15pt, inner sep=2pt] at (0,0) (d) {$d$};
    \node[draw, circle, minimum size=15pt, inner sep=2pt, left=of d] (b) {$b$};
    \node[draw, circle, minimum size=15pt, inner sep=2pt, above=of b] (a) {$a$};
    \node[draw, circle, minimum size=15pt, inner sep=2pt, below=of b] (c) {$c$};
    \node[draw, circle, minimum size=15pt, inner sep=2pt, right=of d] (f) {$f$};
    \node[draw, circle, minimum size=15pt, inner sep=2pt, above=of f] (e) {$e$};
    \node[draw, circle, minimum size=15pt, inner sep=2pt, below=of f] (g) {$g$};
    \node[draw, circle, minimum size=15pt, inner sep=2pt, right=of e] (h) {$h$};
    \node[draw, circle, minimum size=15pt, inner sep=2pt, right=of f] (i) {$i$};
    \node[draw, circle, minimum size=15pt, inner sep=2pt, right=of g] (j) {$j$};

    \draw[thick, -{Stealth[scale=1.5]}] (d) -- (c);
    \draw[thick, -{Stealth[scale=1.5]}] (f) -- (d);
    \draw[thick, -{Stealth[scale=1.5]}] (e) -- (f);
    \draw[thick, -{Stealth[scale=1.5]}] (e) -- (h);
\end{tikzpicture}
}}
}
\end{subfigure}
\qquad
\begin{subfigure}[$h$ as the root][t]{
    \centering
    \resizebox{0.25\linewidth}{!}{\resizebox{0.7\linewidth}{!}{%
\begin{tikzpicture}
    \node[draw, circle, minimum size=15pt, inner sep=2pt] at (0,0) (d) {$d$};
    \node[draw, circle, minimum size=15pt, inner sep=2pt, left=of d] (b) {$b$};
    \node[draw, circle, minimum size=15pt, inner sep=2pt, above=of b] (a) {$a$};
    \node[draw, circle, minimum size=15pt, inner sep=2pt, below=of b] (c) {$c$};
    \node[draw, circle, minimum size=15pt, inner sep=2pt, right=of d] (f) {$f$};
    \node[draw, circle, minimum size=15pt, inner sep=2pt, above=of f] (e) {$e$};
    \node[draw, circle, minimum size=15pt, inner sep=2pt, below=of f] (g) {$g$};
    \node[draw, circle, minimum size=15pt, inner sep=2pt, right=of e] (h) {$h$};
    \node[draw, circle, minimum size=15pt, inner sep=2pt, right=of f] (i) {$i$};
    \node[draw, circle, minimum size=15pt, inner sep=2pt, right=of g] (j) {$j$};

    \draw[thick, -{Stealth[scale=1.5]}] (d) -- (c);
    \draw[thick, -{Stealth[scale=1.5]}] (f) -- (d);
    \draw[thick, -{Stealth[scale=1.5]}] (e) -- (f);
    \draw[thick, -{Stealth[scale=1.5]}] (h) -- (e);
\end{tikzpicture}
}}
}
\end{subfigure}
\caption{
The five different possible orientations of $H$.
Observe that the ground truth orientation of these edges is inconsistent with all five orientations shown here.
}
\label{fig:all-five-possible-H-orientations}
\end{figure}

\section{Skeleton assumption}
\label{sec:skeleton-assumption}

In this section, we present a set of \emph{sufficient} assumptions (\cref{assumption:chow-liu-recovers-correct-skeleton}) under which the Chow-Liu algorithm will recover the true skeleton even with finite samples.
We note that the conditions listed here are in spirit very similar to the assumptions made to recover exact graphical structures in other works \citep{gao2021efficient, DBLP:conf/nips/GhoshalH17, DBLP:conf/aistats/GaoTA22}, i.e., assuming a sufficiently detectable gap on an edge or from an alternate graph. Otherwise, it is not hard to find counter examples to thwart learners from recovering the correct network structure with finite sample access.\footnote{E.g., a distribution on $X \rightarrow Y$ with infinitely small $I(X; Y)$ and given finite sample access no algorithm can distinguish the actual graph from the empty graph.} As such, it is often necessary to make these assumptions for exact structure recovery. Aside from the ones presented here, \citet{DBLP:conf/isit/BankH20, DBLP:journals/corr/abs-2209-07028} study other sufficient conditions for recovering the skeleton of Polytrees (or Bayes nets).

Nevertheless, we would like to highlight that our paper has made progress in polytree PAC-learning in the following statisical sense: it suffices to have exact first order mutual information and approximate higher order mutual information to learn (most) bounded in-degree polytrees in polynomial time. For prior works, it is only known that one can recover polytrees efficiently with exact first and second order mutual information~\citep{DBLP:journals/corr/abs-1304-2736} or exponential time algorithm for approximating bounded in-degree Bayes nets~\citep{DBLP:journals/air/KitsonCG0C23}.

\begin{assumption}
\label{assumption:chow-liu-recovers-correct-skeleton}
For any given distribution $\pd$, there exists a constant $\eps_{\pd} > 0$ such that:\\
(1) For every pair of nodes $u$ and $v$, if there exists a path $u - \cdots - v$ of length greater than $2$ in $G^*$, then then $I(u ; v) + \eps_{\pd} \leq I(a ; b)$ for every pair of adjacent vertices $a - b$ in the path.\\
(2) For every pair of directly connected nodes $a - b$ in $G^*$, $I(a; b)
  \geq \eps_{\pd}$.
\end{assumption}

Suppose there is a large enough gap of $\eps_{\pd}$ between edges in $G^*$ and edges outside of $G^*$.
Then, with $\cO(1 / \eps_{\pd}^2)$ samples, each estimated mutual information $\hat{I}(a; b)$ will be sufficiently close to the true mutual information $I(a; b)$.
Thus, running the Chow-Liu algorithm (which is maximum spanning tree on the estimated mutual information on each pair of vertices) recovers $\skel(G^*)$.
See \cref{sec:appx-skeleton-assumption} for the full proof.

\begin{restatable}{mylemma}{chowliurecoverscorrectskeleton}
\label{lem:chow-liu-recovers-correct-skeleton}
Under \cref{assumption:chow-liu-recovers-correct-skeleton}, running the Chow-Liu algorithm on the $m$-sample empirical estimates $\{\hat{I}(u;v)\}_{u,v \in V}$ recovers a ground truth skeleton with high probability when $m \geq \Omega (\frac{\log n}{\eps_{\pd}^2})$.
\end{restatable}

Combining \cref{lem:chow-liu-recovers-correct-skeleton} with our algorithm \cref{alg:known-skeleton-and-max-in-degree}, one can learn a polytree that is $\eps$-close in KL with $\tilde{\cO} \left( \max \left\{ \frac{\log(n)}{\eps_{\pd}^2}, \frac{2^{d} \cdot n}{\eps} \right\} \right)$ samples, where $\eps_{\pd}$ depends on the distribution $\pd$.

\section{Lower bound}
\label{sec:lower-bound}

In this section, we show that $\Omega(n/\eps)$ samples are necessary \emph{even when a known skeleton is provided}.
For constant in-degree $d$, this shows that our proposed algorithm in \cref{sec:known-skeleton-and-max-in-degree} is sample-optimal up to logarithmic factors.

We first begin by showing a lower bound of $\Omega(1/\eps)$ on a graph with three vertices, even when the skeleton is given.
Let $G_1$ be $X \to Z \to Y$ and $G_2$ be $X \to Z \gets Y$, such that $\skel(G_1) = \skel(G_2)$ is $X - Z - Y$.
Letting $Bern(1/2)$ denote the Bernoulli distribution with parameter 1/2, i.e.\ a fair coin flip, we define $\pd_1$ and $\pd_2$ as follows:

\begin{equation}
\label{eq:lower-bound-distributions}
\pd_1:
\begin{cases}
X \sim \Bern(1/2)\\
Z = \begin{cases}
X & \text{w.p. $1/2$}\\
\Bern(1/2) & \text{w.p. $1/2$}
\end{cases}\\
Y =
\begin{cases}
Z & \text{w.p. $\sqrt{\eps}$}\\
\Bern(1/2) & \text{w.p. $1-\sqrt{\eps}$}
\end{cases}
\end{cases}
\;\hfill\;
\pd_2:
\begin{cases}
X \sim \Bern(1/2)\\
Y \sim \Bern(1/2)\\
Z =
\begin{cases}
X & \text{w.p. $1/2$}\\
Y & \text{w.p. $\sqrt{\eps}$}\\
\Bern(1/2) & \text{w.p. $1/2 - \sqrt{\eps}$}
\end{cases}
\end{cases}
\end{equation}

The intuition is that we keep the edge $X \to Z$ ``roughly the same'' and tweak the edge $Y - Z$ between the distributions.
By defining $P_{i,G}$ as projecting $P_i$ onto $G$, one can show \cref{lem:lower-bound-three-variable-lemma}; see \cref{sec:lower-bound-lemma} for its proof.

\begin{restatable}[Key lower bound lemma]{mylemma}{lowerboundthreevariablelemma}
\label{lem:lower-bound-three-variable-lemma}
Let $G_1$ be $X \to Z \to Y$ and $G_2$ be $X \to Z \gets Y$, such that $\skel(G_1) = \skel(G_2)$ is $X - Z - Y$.
With respect to \cref{eq:lower-bound-distributions}, we have the following:
\begin{enumerate}
    \item $\sqhell(\pd_1, \pd_2) \in \cO(\eps)$
    \item $\kl(\pd_1 \;\|\; P_{1,G_1}) = 0$ and $\kl(\pd_1 \;\|\; P_{1,G_2}) \in \Omega(\eps)$
    \item $\kl(\pd_2 \;\|\; P_{2,G_2}) = 0$ and $\kl(\pd_2 \;\|\; P_{2,G_1}) \in \Omega(\eps)$
\end{enumerate}
\end{restatable}

Our hardness result (\cref{lem:lower-bound-for-three-variables}) is obtained by reducing the problem of finding an $\eps$-close graph orientation of $X - Z - Y$ to the problem of \emph{testing} whether the samples are drawn from $\pd_1$ or $\pd_2$:
To ensure $\eps$-closeness in the graph orientation, one has to correctly determine whether the samples come from $\pd_1$ or $\pd_2$ and then pick $G_1$ or $G_2$ respectively.
Put differently, if one can solve the problem in \cref{lem:lower-bound-for-three-variables}, then one can use that algorithm to solve the problem in \cref{lem:lower-bound-three-variable-lemma}.
However, it is well-known that distinguishing two distributions whose squared Hellinger distance is $\eps$ requires $\Omega(1/\eps)$ samples (e.g.\ see \cite[Theorem 4.7]{BarYossef:02}).

\begin{restatable}{mylemma}{lowerboundforthreevariables}
\label{lem:lower-bound-for-three-variables}
Even when given $\skel(G^*)$, it takes $\Omega(1/\eps)$ samples to learn an $\eps$-close graph orientation of $G^*$ for distributions on $\{0,1\}^3$.
\end{restatable}

Using the above construction as a gadget, we can obtain a dependency on $n$ in our lower bound by constructing $n/3$ independent copies of the above gadget, \`{a} la proof strategy of \citet[Theorem 7.6]{DBLP:journals/corr/abs-2011-04144}. 
For some constant $c > 0$, we know that a constant $1/c$ fraction of the gadgets will incur an error or more than $\eps/n$ if less than $cn/\eps$ samples are used.
The desired result then follows from the tensorization of KL divergence, i.e., $\kl \left( \prod_i P_i \;\|\; \prod_i Q_i \right) = \sum_i \kl(P_i \;\|\;  Q_i)$.

\begin{restatable}{theorem}{lowerboundfornvariables}
\label{thm:lower-bound-for-n-variables}
Even when given $\skel(G^*)$, it takes $\Omega(n/\eps)$ samples to learn an $\eps$-close graph orientation of $G^*$ for distributions on $\{0,1\}^n$.
\end{restatable}

\section{Conclusion and discussion}
\label{sec:conclusion}

In this work, we studied the problem of estimating a distribution defined on a $d$-polytree $\pd$ with graph structure $G^*$ using finite observational samples.
We designed and analyzed an efficient algorithm that produces an estimate $\hat{\pd}$ such that $\kl(\pd \;\|\; \hat{\pd}) \leq \eps$ assuming access to $\skel(G^*)$ and $d$.
The skeleton $\skel(G^*)$ is recoverable under \cref{assumption:chow-liu-recovers-correct-skeleton} and we show that there is an inherent hardness in the learning problem even under the assumption that $\skel(G^*)$ is given.
For constant $d$, our hardness result shows that our proposed algorithm is sample-optimal up to logarithmic factors.

\medskip

It is natural to ask whether what we can do with access to a false skeleton that is approximately correct (i.e.\ has some orientation close in KL to the ground truth) produced by running the Chow-Liu algorithm on the sample statistics.
However, it is unclear to us why we can hope to design efficient algorithms with provable guarantees in this case for two reasons:
\begin{itemize}
    \item The Chow-Liu algorithm only uses order-1 mutual information while the KL divergence of Equation (1) requires information from order-$d$ mutual information. It is unclear why one can hope that this false skeleton would yield provable guarantees with respect to \cref{eq:mutual_information_decom_chou_liu}.
    \item An ``approximately correct'' skeleton may have potentially unknown number of edges in the skeleton being wrong and we do not see how to design efficient global orientation algorithms using only statistics from the ground truth samples.
\end{itemize}
Without the true skeleton, a ``local algorithm'' (such as ours) can be tricked into some ``local optima'' and it is hard to argue why the output would obtain ``global guarantees'' with respect to the parent sets of \cref{eq:mutual_information_decom_chou_liu}.

Another interesting open question is whether one can extend the hardness result to arbitrary $d \geq 1$, or design more efficient learning algorithms for $d$-polytrees.
In particular, we are unaware of any obstruction a lower bound for $|\Sigma| > 2$ and $d > 2$.
While we do not know an optimal construction, the following construction (emulating Appendix A.2 of \cite{CanonneDKS20}) yields $\Omega(\frac{n 2^d}{(d+1) \varepsilon^2})$, showing that the exponential dependence on $d$ is unavoidable. %
Consider $\frac{n}{d+1}$ stars with binary alphabets, where each star center has $d$ incoming parents.
Each parental node is set to be an independent uniform coin flip over the binary alphabet and so it takes $\Omega(2^d/\varepsilon^2)$ to learn each star to accuracy $\varepsilon$.
As KL is additive, one would require any constant fraction of the stars to incur less than $\frac{\varepsilon (d+1)}{n}$ error.
To do so, one would need $\Omega(\frac{n 2^d}{(d+1) \varepsilon^2})$ samples.
\acks{
This research/project is supported by the National Research Foundation, Singapore under its AI Singapore Programme (AISG Award No: AISG-PhD/2021-08-013).
CC was supported by an ARC DECRA (DE230101329) and an unrestricted gift from Google Research.
Joy is supported by the JD Technology Scholarship.
We would like to thank Vipul Arora for valuable feedback and discussions.
}

\bibliography{refs}

\newpage
\appendix

\crefalias{section}{appendix} %
\section{Deferred proofs}
\label{sec:appendix-proofs}

\subsection{Adapting the known tester result of \texorpdfstring{\cite{DBLP:journals/corr/abs-2011-04144}}{Bhattacharyya et al.\ (2021)}}
\label{sec:adapting-known-tester-result}

\cref{cor:CMI_tester} is adapted from Theorem 1.3 of \cite{DBLP:journals/corr/abs-2011-04144}.
See \cref{sec:appendix-explicit-constant} for a derivation of a constant $C$ that works.

\begin{restatable}[Conditional MI Tester, {\cite[Theorem 1.3]{DBLP:journals/corr/abs-2011-04144}}]{theorem}{CMItesterthm}
\label{thm:CMI_tester}
Fix any $\eps > 0$.
Let $(X, Y, Z)$ be three random variables over $\Sigma_X, \Sigma_Y, \Sigma_Z$ respectively.
Given the empirical distribution $(\hat{X}, \hat{Y}, \hat{Z})$ over a size $N$ sample of $(X, Y, Z)$, there exists a universal constant $0 < C < 1$ so that for any
\[
N
\geq \Theta \left( \frac{|\Sigma_X| \cdot |\Sigma_Y| \cdot |\Sigma_Z|}{\eps} \cdot \log \frac{|\Sigma_X| \cdot |\Sigma_Y| \cdot |\Sigma_Z|}{\delta} \cdot \log \frac{|\Sigma_X| \cdot |\Sigma_Y| \cdot |\Sigma_Z| \cdot \log(1 / \delta)}{\eps} \right),
\]
the following results hold with probability $1 - \delta$:
\begin{enumerate}
    \item\label{item:equal_case_CMI_tester} If $I(X; Y \mid Z) = 0$, then $I(\hat{X}; \hat{Y} \mid \hat{Z}) < \eps$.
    \item\label{item:farness_case_CMI_tester} If $I(X; Y \mid Z) \geq \eps$, then $I(\hat{X}; \hat{Y} \mid \hat{Z}) > C \cdot I(X; Y \mid Z)$.
\end{enumerate}
\end{restatable}

In our notation, we use $\hat{I}(X;Y \mid Z)$ to mean the MI of the empirical distribution $I(\hat{X}; \hat{Y} \mid \hat{Z})$.

\CMItester*
\begin{proof}
In the original proof of (\ref{item:equal_case_CMI_tester}) in {\cite[Theorem 1.3]{DBLP:journals/corr/abs-2011-04144}}, it is possible to change $\eps$ to $C \cdot \eps$ by paying a factor $1/C$ more in sample complexity, yielding our first statement.

Now, suppose $\hat{I}(X; Y \mid Z) \leq C \cdot \eps$.
Assume, for a contradiction, that $I(X; Y \mid Z) \geq \eps$.
Then, statement \ref{item:farness_case_CMI_tester} of \cref{thm:CMI_tester} tells us that $\hat{I}(X; Y \mid Z) > C \cdot I(X; Y \mid Z) \geq C \cdot \eps$.
This contradicts the assumption that $\hat{I}(X; Y \mid Z) \leq C \cdot \eps$.
Therefore, we must have $I(X; Y \mid Z) < \eps$.
\end{proof}

\subsection{Algorithm analysis}
\label{sec:appx-algo-analysis}

The following identity (\cref{lem:useful-identity}) of mutual information and two properties about (conditional) mutual information on a polytree (\cref{lem:parent-CMI-manipulation}) which will be helpful in our proofs later.

\begin{restatable}[A useful identity]{mylemma}{usefulidentity}
\label{lem:useful-identity}
For any variable $v$ and sets $A, B \subseteq V \setminus \{v\}$, we have
\[
I(v; A \cup B)
= I(v; A) + I(v; B) + I(A; B \mid v) - I(A;B).
\]
\end{restatable}
\begin{proof}
By the chain rule for mutual information, we can express $I(v, A; B)$ in the following two ways:
\begin{enumerate}
    \item $I(v, A; B) = I(v; B) + I(A; B \mid v)$;
    \item $I(v, A; B) = I(A; B) + I(v; B \mid A)$.
\end{enumerate}
So,
\begin{align*}
I(v; A \cup B)
&= I(v; A) + I(v; B \mid A)\\
&= I(v; A) + I(v, A; B) - I(A; B)\\
&= I(v; A) + I(v; B) + I(A; B \mid v) - I(A; B).
\end{align*}
\end{proof}

\begin{restatable}{mylemma}{parentCMImanipulation}
\label{lem:parent-CMI-manipulation}
Let $v$ be an arbitrary vertex in a Bayesian polytree with parents $\pi(v)$.
Then, we have
\begin{enumerate}
    \item For any disjoint subsets $A, B \subseteq \pi(v)$,
    \[
    I(v; A \cup B) = I(v; A) + I(v; B) + I(A; B \mid v)
    \]
    \item For any subset $A \subseteq \pi(v)$,
    \[
    I(v; A) \geq \sum_{u \in A} I(v; u)
    \]
\end{enumerate}
\end{restatable}
\begin{proof}
For the first equality, apply \cref{lem:useful-identity} by observing that $I(A; B) = 0$ since $A, B \subseteq \pi(v)$.

For the second inequality, apply the first equality $|A|$ times with the observation that conditional mutual information is non-negative.
Suppose $A = \{a_1, \ldots, a_k\}$.
Then,
\begin{align*}
I(v; A)
&= I(v; \{a_1\}) + I(v; A \setminus \{a_1\}) + I(\{a_1\}; A \setminus \{a_1\} \mid v)\\
&\geq I(v; \{a_1\}) + I(v; A \setminus \{a_1\})\\
& \ldots\\
&\geq \sum_{u \in A} I(v; u)
\end{align*}
\end{proof}

\orientedarcsinPhaseonearegroundtruthorientations*
\begin{proof}
We consider the three cases in which we orient edges within the while loop:
\begin{enumerate}
    \item Strong v-structures (in Phase 1)
    \item Forced orientation due to local checks (in Phase 2)
    \item Forced orientation due to Meek $R1(d)$ (in Phase 2)
\end{enumerate}

\paragraph{Case 1: Strong v-structures}

Consider an arbitrary strong deg-$d$ v-structure with center $v$. That is,
there is a set $S$ (all neighbors of $v$) with size $|S| = d$, such that $\hat{I}(u ; S \setminus \{u\} \mid v) \geq C \cdot \eps$ for any $u \in S$.
So, by \cref{cor:CMI_tester}, we know that $I(u; S \setminus \{u\} \mid v) > 0$ for all $u \in S$.

Consider an arbitrary vertex $u_0 \in S$.
Suppose, for a contradiction, that the ground truth orients \emph{some} edge outwards from $v$, say $v \to u_0$ for some $u_0 \in S$.
This would imply that $I(u_0; S \setminus \{u_0\} \mid v) = 0$.
This contradicts the fact that we had $I (u_0 ; S \setminus \{u_0\} \mid v) > 0$ for any $u \in S$.
Therefore, for all $u \in S$, orienting $u \to v$ is a ground truth orientation.

\paragraph{Case 2: Forced orientation due to local checks}

Consider an arbitrary vertex $v$. Suppose it currently has incoming oriented arcs $N^{in}(v)$ and we are checking for the orientation for an unoriented neighbor $u$ of $v$. By induction, the existing incoming arcs to $v$ are ground truth orientations.

If the ground truth orients $u \to v$, then $I(u; N^{in}(v)) = 0$ and we should have $\hat{I}(u; N^{in}(v)) < C \cdot \eps \leq \eps$ via \cref{cor:CMI_tester}.
Hence, if we detect $\hat{I}(N^{in} (v) ; u) > \eps$, it must be the case that the ground truth orientation is $u \leftarrow v$, which is what we also orient.

Meanwhile, if the ground truth orients $u \leftarrow v$, then $I(u; N^{in}(v) \mid v) = 0$ and we should have $\hat{I}(u; N^{in}(v) \mid v) \leq C \cdot \eps \leq \eps$ via \cref{cor:CMI_tester}.
Hence, if we detect $\hat{I}(u; N^{in}(v) \mid v) > \eps$, it must be the case that the ground truth orientation is $u \to v$, which is what we also orient.

See \cref{fig:running-example}(c) for an illustration.
Note that we may possibly detect both $\hat{I}(u; N^{in}(v)) \leq \eps$ and $\hat{I}(u; N^{in}(v) \mid v) \leq \eps$.
In that case, we leave the edge $u \sim v$ unoriented.

\paragraph{Case 3: Forced orientation due to Meek $R1(d)$}
Meek $R1(d)$ only triggers when there are $d$ incoming arcs to a particular vertex.
Since oriented arcs are inductively ground truth orientations and there are at most $d^* \leq d$ incoming arcs to any vertex, the forced orientations due to Meek $R1(d)$ will always be correct.

\end{proof}

\decompositionispossiblewhenvstructureisabsent*
\begin{proof}
Since $S \cup S' \subseteq \pi(v)$, we see that $I(u; S \cup S' \setminus \{u\}) = 0$.
Furthermore, since $S \neq \emptyset$, Phase 1 guarantees that there exists a vertex $u \in S \cup S'$ such that $\hat{I}(u; S \cup S' \setminus \{u\} \mid v) \leq C \cdot \eps$.
To see why, we need to look at \cref{alg:line:orient_condition} of \cref{alg:phase1} where we check all subsets $T$ of $\pi(v)$ (as well as some other sets) to see if \emph{every} $u \in T$ satisfies $\hat{I}(u; T \backslash \{u\} | v) \geq C \cdot \eps$.
From here, we can see that if a subset $T$ of $\pi(v)$ is not \emph{all} oriented into $v$, then we know that from \cref{alg:line:orient_condition} of \cref{alg:phase1} that there exists some $u \in T$ such that $\hat{I}(u; T \backslash \{u\} | v) < C \cdot \eps$.
Applying this to $T = S \cup S'$, where the set of unoriented neighboring nodes $S$ is non-empty, we have our claim.
As $\hat{I}(u; S \cup S' \setminus \{u\} \mid v) \leq C \cdot \eps$, \cref{cor:CMI_tester} tells us that $I(u; S \cup S' \setminus \{u\}) < \eps$, we get
\begin{align*}
  &\; I(v ; S \cup S')\\
  = &\; I(v ; S \cup S' \setminus \{u\}) + I(v ; u) + I(u ; S \cup S' \setminus \{u\} \mid v) - I (u ; S \cup S' \setminus \{u\})\\
  = &\; I(v ; S \cup S' \setminus \{u\}) + I(v ; u) + I(u ; S \cup S' \setminus \{u\} \mid v)\\
  \leq &\; I(v ; S \cup S' \setminus \{u\}) + I(v ; u) + \eps
\end{align*}
\end{proof}

\decompositionsequenceexistsformissingvstructure*
\begin{proof}
Initializing $S' = \pi^{in}(v)$ and $S = \pi(v) \setminus \pi^{in}(v) = \pi^{un}(v)$, we can repeatedly apply \cref{lem:decomposition_is_possible_when_v_structure_is_absent} to remove vertices one by one, until $S = \emptyset$.
Without loss of generality, by relabelling the vertices, we may assume that \cref{lem:decomposition_is_possible_when_v_structure_is_absent} removes $u_1$, then $u_2$, and so on.
Let us denote the set of all removed vertices by $U$ and note that some of the removed vertices may come from $S' = \pi^{in}(v)$.

\begin{align*}
I(v ; \pi(v))
&\leq I(v ; \pi(v) \setminus \{u_1\}) + I(v ; u_1) + \eps && \text{By \cref{lem:decomposition_is_possible_when_v_structure_is_absent}}\\
&\leq I(v ; \pi(v) \setminus \{u_1, u_2\}) + I(v ; u_1) + I(v ; u_2) + 2 \eps && \text{By \cref{lem:decomposition_is_possible_when_v_structure_is_absent}}\\
&\leq \ldots\\
&\leq I(v ; \pi(v) \setminus U) + \sum_{u \in U} I(v ; u) + \eps \cdot |U| && \text{By \cref{lem:decomposition_is_possible_when_v_structure_is_absent}}
\end{align*}

Since $I(A; B) = 0$ for any $A \;\dot\cup\; B \subseteq \pi^{in}(v)$, we have
\begin{align*}
&\; I(v; \pi^{in}(v))\\
= &\; I(v; \pi^{in}(v) \setminus U) + I(v; \pi^{in}(v) \cap U) + I(\pi^{in}(v) \cap U ; \pi^{in}(v) \setminus U \mid v) && \text{By \cref{lem:parent-CMI-manipulation}}\\
\geq &\; I(v; \pi^{in}(v) \setminus U) + I(v; \pi^{in}(v) \cap U)\\
\geq &\; I(v; \pi^{in}(v) \setminus U) + \sum_{u \in \pi^{in}(v) \cap U} I(v;u) && \text{By \cref{lem:parent-CMI-manipulation}}
\end{align*}
where the second last inequality is because $I(\pi^{in}(v) ; \pi^{in}(v) \cap U \mid v) \geq 0$.

\begin{align*}
I(v ; \pi(v))
&\leq I(v ; \pi^{in}(v) \setminus U) + \sum_{u \in U} I(v ; u) + \eps \cdot |U| && \text{From above}\\
&= I(v ; \pi^{in}(v) \setminus U) + \sum_{u \in \pi^{in}(v) \cap U} I(v ; u) + \sum_{u \in \pi^{un}(v)} I(v ; u) + \eps \cdot |U| && \text{Since $\pi^{un} \subseteq U$}\\
&\leq I(v; \pi^{in}(v)) + \sum_{u \in \pi^{un}(v)} I(v ; u) + \eps \cdot |U| && \text{From above}\\
&\leq I(v; \pi^{in}(v)) + \sum_{u \in \pi^{un}(v)} I(v ; u) + \eps \cdot |\pi(v)| && \text{Since $U \subseteq \pi(v)$}
\end{align*}
\end{proof}

\singleextraincomingisokay*
\begin{proof}
Since $u \sim v \in E(H)$ remained unoriented, Phase 2 guarantees that $\hat{I}(u; \pi^{in}(v) \mid v) \leq \eps$ and $\hat{I}(u; \pi^{in}(v)) \leq \eps$.
Since $0 < C < 1$, we also see that $\hat{I}(u; \pi^{in}(v) \mid v) \leq C \cdot \eps$ and $\hat{I}(u; \pi^{in}(v)) \leq C \cdot \eps$ and so \cref{cor:CMI_tester} tells us that $I(u; \pi^{in}(v) \mid v) \leq \eps$ and $I(u; \pi^{in}(v)) \leq \eps$.
So,
\begin{align*}
&\; |I(v ; \pi^{in}(v) \cup u) - I(v ; \pi^{in}(v)) - I(v ; u)|\\
= &\; | I(u; \pi^{in}(v) \mid v) - I(u; \pi^{in}(v)) | && \text{By \cref{lem:useful-identity}}\\
= &\; \max\{I(u; \pi^{in}(v) \mid v), I(u; \pi^{in}(v)) \} && \text{At most one of these term can be non-zero}\\
\leq &\; \eps
\end{align*}
\end{proof}

\finaloutputisgood*
\begin{proof}
We will argue that this summation is bounded by individually bounding each term in the summation.
The main argument of the proof is that once we identified all the strong v-structures (and thus cancel out the scores of every strong v-structures), the rest should be roughly the score of a tree (up to additive $\eps$ error).
Then, since we are guaranteed to be given $\skel(G^*)$, the tree scores will match.

Let $A \subseteq V$ be the set of vertices which receive an additional incoming neighbor in the final phase, which we denote by $a_v \in V$, i.e.\ $\hat{\pi}(v) = \pi^{in}(v) \cup \{a_v\}$.
Note that the set of edges $\{a_v \to v\}_{v \in A}$ is exactly the edges of the undirected graph $H$ in the final phase.
See \cref{fig:final-phase} for an illustration.

To lower bound $\sum_{v \in V} I(v ; \hat{\pi}(v))$, we can show
\begin{align*}
&\; \sum_{v \in V} I(v ; \hat{\pi}(v))\\
= &\; \sum_{v \in A} I(v ; \hat{\pi}(v)) + \sum_{v \in V \setminus A} I(v ; \hat{\pi}(v))\\
\geq &\; \sum_{v \in A} \Big( I(v; \hat{\pi}(v) \setminus \{a_v\}) + I(v; a_v) - \eps \Big) + \sum_{v \in V \setminus A} I(v ; \hat{\pi}(v)) && \text{By \cref{lemma:single-extra-incoming-is-okay}}\\
= &\; \sum_{v \in A} I(v; \pi^{in}(v)) + \sum_{v \in A} I(v; a_v) + \sum_{v \in V \setminus A} I(v ; \pi^{in}(v)) - |A| \cdot \eps\\
= &\; \sum_{v \in V} I(v; \pi^{in}(v)) + \sum_{v \in A} I(v; a_v) - |A| \cdot \eps\\
\geq &\; \sum_{v \in V} I(v; \pi^{in}(v)) + \sum_{v \in A} I(v; a_v) - n \eps && \text{Since $A \subseteq V$ and $|V| = n$}
\end{align*}

Meanwhile, to upper bound $\sum_{v \in V} I(v ; \pi(v))$, we can show
\begin{align*}
&\; \sum_{v \in V} I(v ; \pi(v))\\
= &\; \sum_{\substack{v \in V\\ \pi^{un}(v) \neq \emptyset}} I(v ; \pi(v)) + \sum_{\substack{v \in V\\ \pi^{un}(v) = \emptyset}} I(v ; \pi(v))\\
\leq &\; \sum_{\substack{v \in V\\ \pi^{un}(v) \neq \emptyset}} \left( \eps \cdot |\pi(v)| + I (v ; \pi^{in}(v)) + \sum_{u \in \pi^{un}(v)} I (v ; u) \right) + \sum_{\substack{v \in V\\ \pi^{un}(v) = \emptyset}} I(v ; \pi(v)) && \text{By \cref{lem:decomposition_sequence_exists_for_missing_v_structure}}\\
= &\; \sum_{v \in V} I(v ; \pi^{in}(v)) + \sum_{\substack{v \in V\\ \pi^{un}(v) \neq \emptyset}} \left( \eps \cdot |\pi(v)| + \sum_{u \in \pi^{un}(v)} I (v ; u) \right)
\end{align*}
where the final equality is because $\pi^{in}(v) = \pi(v)$ when $\pi^{un}(v) = \emptyset$.
Since $|\pi(v)| \leq d^*$ and $|V| = n$, we get
\[
\sum_{v \in V} I(v ; \pi(v))
\leq n d^* \eps + \sum_{v \in V} I(v ; \pi^{in}(v)) + \sum_{\substack{v \in V\\ \pi^{in}(v) \neq \emptyset}} \sum_{u \in \pi^{un}(v)} I (v ; u)
\]

Putting together, we get
\begin{align*}
&\; \sum_{v \in V} I(v; \pi (v)) - \sum_{v \in V} I(v ; \hat{\pi} (v))\\
\leq &\; \left( n d^* \eps + \sum_{v \in V} I(v ; \pi^{in}(v)) + \sum_{\substack{v \in V\\ \pi^{in}(v) \neq \emptyset}} \sum_{u \in \pi^{un}(v)} I (v ; u) \right) && \text{From above}\\
&\; - \left( \sum_{v \in V} I(v; \pi^{in} (v)) + \sum_{v \in A} I(v; a_v) - n \eps \right) & \\
= &\; n \cdot (d^* + 1) \cdot \eps + \sum_{v \in V} \sum_{u \in \pi^{un}(v)} I(v; u) - \sum_{v \in A} I (v ; a_v)\\
= &\; n \cdot (d^* + 1) \cdot \eps
\end{align*}
where the last equality is because the last two terms are two different ways to enumerate the edges of $H$, e.g.\ see \cref{fig:final-phase}.
\end{proof}

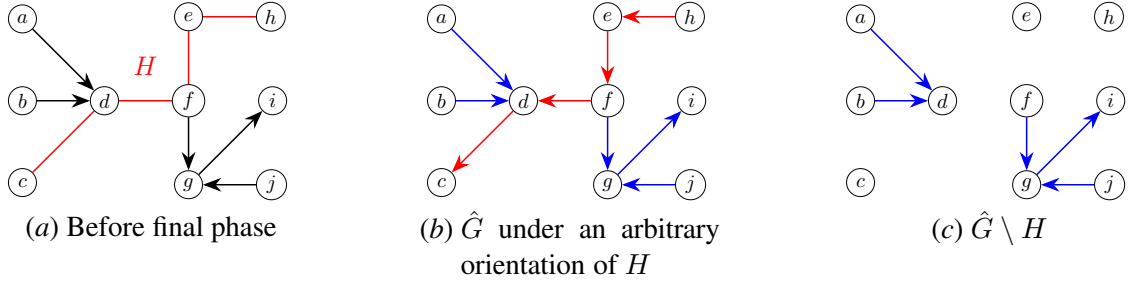
\begin{figure}[htb]
\centering
\begin{subfigure}[Before final phase][t]{
    \centering
    \resizebox{0.25\linewidth}{!}{\resizebox{0.7\linewidth}{!}{%
\begin{tikzpicture}
    \node[draw, circle, minimum size=15pt, inner sep=2pt] at (0,0) (d) {$d$};
    \node[draw, circle, minimum size=15pt, inner sep=2pt, left=of d] (b) {$b$};
    \node[draw, circle, minimum size=15pt, inner sep=2pt, above=of b] (a) {$a$};
    \node[draw, circle, minimum size=15pt, inner sep=2pt, below=of b] (c) {$c$};
    \node[draw, circle, minimum size=15pt, inner sep=2pt, right=of d] (f) {$f$};
    \node[draw, circle, minimum size=15pt, inner sep=2pt, above=of f] (e) {$e$};
    \node[draw, circle, minimum size=15pt, inner sep=2pt, below=of f] (g) {$g$};
    \node[draw, circle, minimum size=15pt, inner sep=2pt, right=of e] (h) {$h$};
    \node[draw, circle, minimum size=15pt, inner sep=2pt, right=of f] (i) {$i$};
    \node[draw, circle, minimum size=15pt, inner sep=2pt, right=of g] (j) {$j$};

    \draw[thick, -{Stealth[scale=1.5]}] (a) -- (d);
    \draw[thick, -{Stealth[scale=1.5]}] (b) -- (d);
    \draw[thick, red] (c) -- (d);
    \draw[thick, red] (d) -- node[above=10pt,midway] {\Large $H$} (f);
    \draw[thick, red] (e) -- (f);
    \draw[thick, -{Stealth[scale=1.5]}] (f) -- (g);
    \draw[thick, red] (e) -- (h);
    \draw[thick, -{Stealth[scale=1.5]}] (g) -- (i);
    \draw[thick, -{Stealth[scale=1.5]}] (j) -- (g);
\end{tikzpicture}
}}
}
\end{subfigure}
\hfill
\begin{subfigure}[$\hat{G}$ under an arbitrary orientation of $H$][t]{
    \centering
    \resizebox{0.25\linewidth}{!}{\resizebox{0.7\linewidth}{!}{%
\begin{tikzpicture}
    \node[draw, circle, minimum size=15pt, inner sep=2pt] at (0,0) (d) {$d$};
    \node[draw, circle, minimum size=15pt, inner sep=2pt, left=of d] (b) {$b$};
    \node[draw, circle, minimum size=15pt, inner sep=2pt, above=of b] (a) {$a$};
    \node[draw, circle, minimum size=15pt, inner sep=2pt, below=of b] (c) {$c$};
    \node[draw, circle, minimum size=15pt, inner sep=2pt, right=of d] (f) {$f$};
    \node[draw, circle, minimum size=15pt, inner sep=2pt, above=of f] (e) {$e$};
    \node[draw, circle, minimum size=15pt, inner sep=2pt, below=of f] (g) {$g$};
    \node[draw, circle, minimum size=15pt, inner sep=2pt, right=of e] (h) {$h$};
    \node[draw, circle, minimum size=15pt, inner sep=2pt, right=of f] (i) {$i$};
    \node[draw, circle, minimum size=15pt, inner sep=2pt, right=of g] (j) {$j$};

    \draw[thick, -{Stealth[scale=1.5]}, blue] (a) -- (d);
    \draw[thick, -{Stealth[scale=1.5]}, blue] (b) -- (d);
    \draw[thick, -{Stealth[scale=1.5]}, red] (d) -- (c);
    \draw[thick, -{Stealth[scale=1.5]}, red] (f) -- (d);
    \draw[thick, -{Stealth[scale=1.5]}, red] (e) -- (f);
    \draw[thick, -{Stealth[scale=1.5]}, blue] (f) -- (g);
    \draw[thick, -{Stealth[scale=1.5]}, red] (h) -- (e);
    \draw[thick, -{Stealth[scale=1.5]}, blue] (g) -- (i);
    \draw[thick, -{Stealth[scale=1.5]}, blue] (j) -- (g);
\end{tikzpicture}
}}
}
\end{subfigure}
\hfill
\begin{subfigure}[$\hat{G} \setminus H$][t]{
    \centering
    \resizebox{0.25\linewidth}{!}{\resizebox{0.7\linewidth}{!}{%
\begin{tikzpicture}
    \node[draw, circle, minimum size=15pt, inner sep=2pt] at (0,0) (d) {$d$};
    \node[draw, circle, minimum size=15pt, inner sep=2pt, left=of d] (b) {$b$};
    \node[draw, circle, minimum size=15pt, inner sep=2pt, above=of b] (a) {$a$};
    \node[draw, circle, minimum size=15pt, inner sep=2pt, below=of b] (c) {$c$};
    \node[draw, circle, minimum size=15pt, inner sep=2pt, right=of d] (f) {$f$};
    \node[draw, circle, minimum size=15pt, inner sep=2pt, above=of f] (e) {$e$};
    \node[draw, circle, minimum size=15pt, inner sep=2pt, below=of f] (g) {$g$};
    \node[draw, circle, minimum size=15pt, inner sep=2pt, right=of e] (h) {$h$};
    \node[draw, circle, minimum size=15pt, inner sep=2pt, right=of f] (i) {$i$};
    \node[draw, circle, minimum size=15pt, inner sep=2pt, right=of g] (j) {$j$};

    \draw[thick, -{Stealth[scale=1.5]}, blue] (a) -- (d);
    \draw[thick, -{Stealth[scale=1.5]}, blue] (b) -- (d);
    \draw[thick, -{Stealth[scale=1.5]}, blue] (f) -- (g);
    \draw[thick, -{Stealth[scale=1.5]}, blue] (g) -- (i);
    \draw[thick, -{Stealth[scale=1.5]}, blue] (j) -- (g);
\end{tikzpicture}
}}
}
\end{subfigure}
\caption{
Illustration of notation used in proof of \cref{lem:final-output-is-good}.
Suppose (a) is the partial orientation of \cref{fig:running-example} after Phase 2, with $H$ as the edge-induced subgraph on the unoriented edges in red.
Before the final phase, we have $\pi^{\mathrm{in}}(d) = \{a,b\}$, $\pi^{\mathrm{in}}(g) = \{f,j\}$, $\pi^{\mathrm{in}}(i) = \{g\}$, $\pi^{\mathrm{un}}(c) = \{d\}$, $\pi^{\mathrm{un}}(d) = \{c,f\}$, $\pi^{\mathrm{un}}(f) = \{d,e\}$, $\pi^{\mathrm{un}}(e) = \{h,f\}$, and $\pi^{\mathrm{un}}(h) = \{e\}$.
With respect to $H$'s orientation in (b), we have $A = \{c,d,f,e,h\}$, $a_c = d$, $a_d = f$, $a_f = e$, and $a_e = h$.
Observe that the $\pi^{\mathrm{un}}$s and $a_{\square}$s are two different ways to refer to the red edges and (b) only shows one possible orientation of $H$ (see \cref{fig:all-five-possible-H-orientations} for others).
}
\label{fig:final-phase}
\end{figure}

\subsection{Skeleton assumption}
\label{sec:appx-skeleton-assumption}

\chowliurecoverscorrectskeleton*
\begin{proof}
Fix a graph $G^*$.
Recall that the Chow-Liu algorithm can be thought of as running maximum spanning tree with the edge weights as the estimated mutual information between any pair of vertices.
With $m \geq \Omega(\log(n)/\eps_{\pd}^2)$ samples and \cref{assumption:chow-liu-recovers-correct-skeleton}, one can estimate $\hat{I}(u;v)$ up to $(\eps_{\pd})/3$-closeness with high probability in $n$, i.e.\ $| I(u;v) - \hat{I}(u;v) | \leq \eps_{\pd}/3$ for any pair of vertices $u,v \in V$.

Now, consider two arbitrary vertices $u$ and $v$ that are \emph{not} neighbors in $G^*$.

\paragraph{Case 1 ($u$ and $v$ belong in the same connected component in $G^*$):}
Let $P_{u,v} = z_0 - z_1 - \ldots - z_k - z_{k+1}$ be the unique path between $u = z_0$ and $v = z_{k+1}$, where $k \geq 1$.
Then,
\[
\hat{I}(u;v) - \eps_{\pd}/3 \leq I(u;v) \leq I(z_i, z_{i+1}) - \eps_{\pd} \leq \hat{I}(z_i, z_{i+1}) - 2 \cdot \eps_{\pd}/3
\]
for any $i \in \{1, \ldots, k\}$.
Since $\hat{I}(u;v) \leq \hat{I}(z_i, z_{i+1}) - \eps_{\pd}/3$ for each $i \in \{1, \ldots, k\}$, the Chow-Liu algorithm will \emph{not} add the edge $u \sim v$ in the output tree.

\paragraph{Case 2 ($u$ and $v$ belong in the different connected components in $G^*$):}
Since $u$ and $v$ belong in the different connected components in $G^*$, we have $I(u;v) = 0$.
With $m$ samples, for any two edge $a \sim b$ in $G^*$, we have
\[
\hat{I}(u;v)
\leq I(u;v) + \eps_{\pd}/3
= \eps_{\pd}/3
< 2 \cdot \eps_{\pd}/3
\leq I(a;b) - \eps_{\pd}/3
\leq \hat{I}(a;b)
\]
That is, the Chow-Liu algorithm will always consider edges crossing different components \emph{after} all true edges have been considered.
\end{proof}

\subsection{Proof of key lower bound lemma}
\label{sec:lower-bound-lemma}

We will use the following inequality in our proofs.

\begin{restatable}{fact}{loginequality}
\label{fact:log-inequality}
For $x > 0$, we have $\log_2(1+x) \geq \log_2(e) \cdot \left( x - \frac{x^2}{2} \right) = \log_2(e) \cdot x \cdot \left(1 - \frac{x}{2} \right)$.
\end{restatable}

Recall the lower bound distributions from \cref{sec:lower-bound}, but we replace $\sqrt{\eps}$ with $\alpha$ for notational convenience:

\begin{equation*}
\pd_1:
\begin{cases}
X \sim \Bern(1/2)\\
Z = \begin{cases}
X & \text{w.p. 1/2}\\
\Bern(1/2) & \text{w.p. 1/2}
\end{cases}\\
Y =
\begin{cases}
Z & \text{w.p. $\alpha$}\\
\Bern(1/2) & \text{w.p. $1 - \alpha$}
\end{cases}
\end{cases}
\;\hfill\;
\pd_2:
\begin{cases}
X \sim \Bern(1/2)\\
Y \sim \Bern(1/2)\\
Z =
\begin{cases}
X & \text{w.p. 1/2}\\
Y & \text{w.p. $\alpha$}\\
\Bern(1/2) & \text{w.p. $1/2 - \alpha$}
\end{cases}
\end{cases}
\end{equation*}

By construction, we have
\[
\pd_1(x, y, z)
= \frac{1}{2} \cdot \left( \frac{1}{2} \cdot \frac{1}{2} + \frac{1}{2} \cdot \mathbbm{1}_{x = z} \right) \cdot \left( \alpha \cdot \mathbbm{1}_{y = z} + (1 - \alpha) \cdot \frac{1}{2} \right)
\]
and
\[
\pd_2(x, y, z)
= \frac{1}{2} \cdot \frac{1}{2} \cdot \left( \frac{1}{2} \cdot \mathbbm{1}_{x = z} + \alpha \cdot \mathbbm{1}_{y = z} + \left( \frac{1}{2} - \alpha \right) \cdot \frac{1}{2} \right)
\]
which corresponds to the probability tables given in \cref{tab:all-probability-tables}.

\begin{table}[htb]
\centering
\begin{tabular}{ccccc}
\toprule
$x$ & $y$ & $z$ & $\pd_1(x,y,z)$ & $\pd_2(x,y,z)$\\
\midrule
$0$ & $0$ & $0$ & $\frac{3}{16} \cdot (1 + \alpha)$ & $\frac{1}{16} \cdot (3 + 2 \alpha)$\\
$0$ & $0$ & $1$ & $\frac{1}{16} \cdot (1 - \alpha)$ & $\frac{1}{16} \cdot (1 - 2 \alpha)$\\
$0$ & $1$ & $0$ & $\frac{3}{16} \cdot (1 - \alpha)$ & $\frac{1}{16} \cdot (3 - 2 \alpha)$\\
$0$ & $1$ & $1$ & $\frac{1}{16} \cdot (1 + \alpha)$ & $\frac{1}{16} \cdot (1 + 2 \alpha)$\\
$1$ & $0$ & $0$ & $\frac{1}{16} \cdot (1 + \alpha)$ & $\frac{1}{16} \cdot (1 + 2 \alpha)$\\
$1$ & $0$ & $1$ & $\frac{3}{16} \cdot (1 - \alpha)$ & $\frac{1}{16} \cdot (3 - 2 \alpha)$\\
$1$ & $1$ & $0$ & $\frac{1}{16} \cdot (1 - \alpha)$ & $\frac{1}{16} \cdot (1 - 2 \alpha)$\\
$1$ & $1$ & $1$ & $\frac{3}{16} \cdot (1 + \alpha)$ & $\frac{1}{16} \cdot (3 + 2 \alpha)$\\
\bottomrule\\
\end{tabular}
\\
\begin{tabular}{ccc}
\toprule
$x$ & $y$ & $\pd_1(x,y)$\\
\midrule
$0$ & $0$ & $\frac{1}{8} \cdot (2 + \alpha)$\\
$0$ & $1$ & $\frac{1}{8} \cdot (2 - \alpha)$\\
$1$ & $0$ & $\frac{1}{8} \cdot (2 - \alpha)$\\
$1$ & $1$ & $\frac{1}{8} \cdot (2 + \alpha)$\\
\bottomrule\\
\end{tabular}
\quad
\begin{tabular}{ccccc}
\toprule
$x$ & $y$ & $\pd_2(x,y \mid z = 0)$ & $\pd_2(x,y \mid z = 1)$\\
\midrule
$0$ & $0$ & $\frac{1}{8} \cdot (3 + 2 \alpha)$ & $\frac{1}{8} \cdot (1 - 2 \alpha)$\\
$0$ & $1$ & $\frac{1}{8} \cdot (3 - 2 \alpha)$ & $\frac{1}{8} \cdot (1 + 2 \alpha)$\\
$1$ & $0$ & $\frac{1}{8} \cdot (1 + 2 \alpha)$ & $\frac{1}{8} \cdot (3 - 2 \alpha)$\\
$1$ & $1$ & $\frac{1}{8} \cdot (1 - 2 \alpha)$ & $\frac{1}{8} \cdot (3 + 2 \alpha)$\\
\bottomrule\\
\end{tabular}
\\
\begin{tabular}{ccc}
\toprule
$x$ & $\pd_2(x \mid z = 0)$ & $\pd_2(x \mid z = 1)$\\
\midrule
$0$ & $\frac{3}{4}$ & $\frac{1}{4}$\\
$1$ & $\frac{1}{4}$ & $\frac{3}{4}$\\
\bottomrule\\
\end{tabular}
\quad
\begin{tabular}{ccc}
\toprule
$y$ & $\pd_2(y \mid z = 0)$ & $\pd_2(y \mid z = 1)$\\
\midrule
$0$ & $\frac{1 + \alpha}{2}$ & $\frac{1 - \alpha}{2}$\\
$1$ & $\frac{1 - \alpha}{2}$ & $\frac{1 + \alpha}{2}$\\
\bottomrule\\
\end{tabular}
\caption{Explicit (conditional) probability tables for our lower bound construction.}
\label{tab:all-probability-tables}
\end{table}

\begin{lemma}
\label{lem:lower-bound-squared-hellinger}
$d^2_H(\pd_1, \pd_2) \leq \alpha^2$
\end{lemma}
\begin{proof}
From \cref{tab:all-probability-tables}, we see that
\begin{align*}
&\; \sum_{(x,y,z) \in \{0,1\}^3} \sqrt{\pd_1(x,y,z) \cdot \pd_2(x,y,z)}\\
= &\; 
\frac{1}{8} \cdot \Big( \sqrt{3 \cdot (1 + \alpha) \cdot (3 + 2 \alpha)} + \sqrt{(1 - \alpha) \cdot (1 - 2 \alpha)}\\
&\; \qquad + \sqrt{3 \cdot (1 - \alpha) \cdot (3 - 2 \alpha)} + \sqrt{(1 + \alpha) \cdot (1 + 2 \alpha)} \Big)\\
\end{align*}

Considering the Taylor expansion of each of the four terms at $\alpha = 0$:
\begin{align*}
\sum_{(x,y,z) \in \{0,1\}^3} \sqrt{\pd_1(x,y,z) \cdot \pd_2(x,y,z)}
\geq
\frac{1}{8} \cdot \left( 8 - \frac{\alpha^2}{3} - \cO(\alpha^4) \right)
\geq
1 - \frac{\alpha^2}{24} - \cO(\alpha^4)
\end{align*}

Hence,
\[
d^2_H(\pd_1, \pd_2)
= 1 - \sum_{(x,y,z) \in \{0,1\}^3} \sqrt{\pd_1(x,y,z) \cdot \pd_2(x,y,z)}
\leq \frac{\alpha^2}{24} + \cO(\alpha^4)
\in \cO(\alpha^2)
\]
\end{proof}

\begin{lemma}
\label{lem:lower-bound-project-one-to-two}
$\kl(\pd_1 \;\|\; P_{1,G_1}) = 0$ and $\kl(\pd_1 \;\|\; P_{1,G_2}) \in \Omega(\alpha^2)$
\end{lemma}
\begin{proof}
We have $\kl(\pd_1 \;\|\; P_{1,G_1}) = 0$ by definition of $\pd_1$: $Z$ depends on $X$ and $Y$ depends on $Z$.

Observe that

\begin{align*}
&\; \kl(\pd_1 \;\|\; P_{1,G_2})\\
= &\; I(X; Z) + I(Z; Y) - I(Z; X,Y)\\
= &\; I(X; Z) + I(Z; Y) - \Big( I(Z; X) + I(Z; Y) + I(X; Y \mid Z) - I(X; Y) \Big) && (\dag)\\
= &\; I(X;Y) - I(X; Y \mid Z)\\
= &\; I(X;Y) && (\ast)
\end{align*}

where $(\dag)$ is by applying \cref{lem:useful-identity} with $v = Z$, $A = \{X\}$, $B = \{Y\}$ and $(\ast)$ is because $I(X; Y \mid Z) = 0$ in $\pd_1$.

We will now show that $I(X;Y) \in \Omega(\alpha^2)$.
From \cref{tab:all-probability-tables}, one can verify that $\pd_1(x = 0) = \pd_1(x = 1) = \pd_1(y = 0) = \pd_1(y = 1) = 1/2$.
So,
\begin{align*}
I(X;Y)
&= \sum_{(x,y) \in \{0,1\}^2} \pd_1(x,y) \cdot \log \frac{\pd_1(x,y)}{\pd_1(x) \cdot \pd_1(y)}\\
&= \frac{1}{4} \cdot \left( (2 + \alpha) \cdot \log \left( 1 + \frac{\alpha}{2} \right) + (2 - \alpha) \cdot \log \left( 1 - \frac{\alpha}{2} \right) \right) && \text{From \cref{tab:all-probability-tables}}\\
&\geq \frac{1}{4} \cdot \log_2(e) \cdot \frac{\alpha}{2} \cdot \left( (2 + \alpha) \cdot \left( 1 - \frac{\alpha}{4} \right) - (2 - \alpha) \cdot \left( 1 + \frac{\alpha}{4} \right) \right) && \text{By \cref{fact:log-inequality}}\\
&\in \Omega(\alpha^2)
\end{align*}
\end{proof}

\begin{lemma}
\label{lem:lower-bound-project-two-to-one}
$\kl(\pd_2 \;\|\; P_{2,G_2}) = 0$ and $\kl(\pd_2 \;\|\; P_{2,G_1}) \in \Omega(\alpha^2)$
\end{lemma}
\begin{proof}
We have $\kl(\pd_2 \;\|\; P_{2,G_2}) = 0$ by definition of $\pd_2$: $Z$ depends on both $X$ and $Y$.

Observe that
\begin{align*}
&\; \kl(\pd_2 \;\|\; P_{2,G_1})\\
= &\; I(Z;X,Y) - I(X; Z) - I(Z;Y)\\
= &\; \left( I(Z; X) + I(Z; Y) + I(X; Y \mid Z) - I(X; Y) \right) - I(X; Z) - I(Z;Y) && (\dag)\\
= &\; I(X; Y \mid Z) - I(X; Y)\\
= &\; I(X; Y \mid Z) && (\ast)
\end{align*}
where $(\dag)$ is by applying \cref{lem:useful-identity} with $v = Z$, $A = \{X\}$, $B = \{Y\}$  and $(\ast)$ is because $I(X;Y) = 0$ in $\pd_2$.

We will now show that $I(X; Y \mid Z) \in \Omega(\alpha^2)$.
By definition,
\begin{multline*}
I(X; Y \mid Z)
= \sum_{(x,y,z) \in \{0,1\}^3} \pd_2(x,y \mid z) \cdot \log \left( \frac{\pd_2(x,y \mid z)}{\pd_2(x \mid z) \cdot \pd_2(y \mid z)} \right)\\
= I(X; Y \mid Z = 0) + I(X; Y \mid Z = 1)    
\end{multline*}
From \cref{tab:all-probability-tables}, one can verify that $\pd_2(z = 0) = \pd_2(z = 1) = 1/2$ and $I(X; Y \mid Z = 0) = I(X; Y \mid Z = 1)$.
So, it suffices to show that $I(X; Y \mid Z = 0) \in \Omega(\alpha^2)$.

\begin{align*}
&\; I(X; Y \mid Z = 0)\\
= &\; \sum_{(x,y) \in \{0,1\}^2} \pd_2(x,y \mid z = 0) \cdot \log \left( \frac{\pd_2(x,y \mid z = 0)}{\pd_2(x \mid z = 0) \cdot \pd_2(y \mid z = 0)} \right)\\
= &\; \frac{3}{8} \cdot \log \left( \frac{3 + 2 \alpha}{3 + 3 \alpha} \cdot \frac{3 - 2 \alpha}{3 - 3 \alpha} \right) + \frac{1}{8} \cdot \log \left( \frac{1 + 2 \alpha}{1 + \alpha} \cdot \frac{1 - 2 \alpha}{1 - \alpha} \right)\\
&\; \qquad + \frac{\alpha}{4} \cdot \log \left( \frac{3 + 2 \alpha}{3 + 3 \alpha} \cdot \frac{3 - 3 \alpha}{3 - 2 \alpha} \cdot \frac{1 + 2 \alpha}{1 + \alpha} \cdot \frac{1 - \alpha}{1 - 2 \alpha} \right)
\end{align*}

\ignore{
By Taylor series at $\alpha = 0$, one can verify that
\begin{itemize}
    \item By Taylor series\footnote{e.g.\ see \url{https://www.wolframalpha.com/input?i=taylor+series+of+\%5Cfrac\%7B3+\%2B+2+\%5Calpha\%7D\%7B3+\%2B+3+\%5Calpha\%7D+\%5Ccdot+\%5Cfrac\%7B3+-+2+\%5Calpha\%7D\%7B3+-+3+\%5Calpha\%7D+at+alpha+\%3D+0}} of $\frac{3 + 2 \alpha}{3 + 3 \alpha} \cdot \frac{3 - 2 \alpha}{3 - 3 \alpha} = 1 + \frac{5}{9} \alpha^2 + \frac{5}{9} \alpha^4 + \cO(\alpha^6)$
    \item By Taylor series\footnote{e.g.\ see \url{https://www.wolframalpha.com/input?i=taylor+series+of+\%5Cfrac\%7B1+\%2B+2+\%5Calpha\%7D\%7B1+\%2B+\%5Calpha\%7D+\%5Ccdot+\%5Cfrac\%7B1+-+2+\%5Calpha\%7D\%7B1+-+\%5Calpha\%7D+at+alpha+\%3D+0}} of $\frac{1 + 2 \alpha}{1 + \alpha} \cdot \frac{1 - 2 \alpha}{1 - \alpha} = 1 - 3 \alpha^2 - 3 \alpha^4 + \cO(\alpha^6)$
    \item By Taylor series\footnote{e.g.\ see \url{https://www.wolframalpha.com/input?i=taylor+series+of+\%5Cfrac\%7B3+\%2B+2+\%5Calpha\%7D\%7B3+\%2B+3+\%5Calpha\%7D+\%5Ccdot+\%5Cfrac\%7B3+-+3+\%5Calpha\%7D\%7B3+-+2+\%5Calpha\%7D+\%5Ccdot+\%5Cfrac\%7B1+\%2B+2+\%5Calpha\%7D\%7B1+\%2B+\%5Calpha\%7D+\%5Ccdot+\%5Cfrac\%7B1+-+\%5Calpha\%7D\%7B1+-+2+\%5Calpha\%7D+at+0}} of $\frac{3 + 2 \alpha}{3 + 3 \alpha} \cdot \frac{3 - 3 \alpha}{3 - 2 \alpha} \cdot \frac{1 + 2 \alpha}{1 + \alpha} \cdot \frac{1 - \alpha}{1 - 2 \alpha}$\\ $= 1 + \frac{4}{3} \alpha + \frac{8}{9} \alpha^2 + \frac{124}{27} \alpha^3 + \frac{464}{81} \alpha^4 + \frac{3844}{243} \alpha^5 + \cO(\alpha^6)$
\end{itemize}
}
Using Taylor series expansion, one can verify that for $0 \leq \alpha \leq 1/2$, 
\begin{itemize}
    \item $\frac{3 + 2 \alpha}{3 + 3 \alpha} \cdot \frac{3 - 2 \alpha}{3 - 3 \alpha} \geq 1 + \frac{5}{9} \alpha^2$
    \item $\frac{1 + 2 \alpha}{1 + \alpha} \cdot \frac{1 - 2 \alpha}{1 - \alpha} \geq 1 - 3 \alpha^2 - 4 \alpha^4$
    \item $\frac{3 + 2 \alpha}{3 + 3 \alpha} \cdot \frac{3 - 3 \alpha}{3 - 2 \alpha} \cdot \frac{1 + 2 \alpha}{1 + \alpha} \cdot \frac{1 - \alpha}{1 - 2 \alpha} \geq 1 + \frac{4}{3} \alpha$
\end{itemize}

Thus, using \cref{fact:log-inequality}, we get
\begin{align*}
&\; I(X; Y \mid Z = 0)\\
= &\; \frac{3}{8} \cdot \log \left( \frac{3 + 2 \alpha}{3 + 3 \alpha} \cdot \frac{3 - 2 \alpha}{3 - 3 \alpha} \right) + \frac{1}{8} \cdot \log \left( \frac{1 + 2 \alpha}{1 + \alpha} \cdot \frac{1 - 2 \alpha}{1 - \alpha} \right)\\
&\; \qquad + \frac{\alpha}{4} \cdot \log \left( \frac{3 + 2 \alpha}{3 + 3 \alpha} \cdot \frac{3 - 3 \alpha}{3 - 2 \alpha} \cdot \frac{1 + 2 \alpha}{1 + \alpha} \cdot \frac{1 - \alpha}{1 - 2 \alpha} \right)\\
\geq &\; \frac{3}{8} \cdot \log \left( 1 + \frac{5}{9} \alpha^2 \right) + \frac{1}{8} \cdot \log \left( 1 - 3 \alpha^2 - 4 \alpha^4 \right) + \frac{\alpha}{4} \cdot \log \left( 1 + \frac{4}{3} \alpha \right)\\
\geq &\; \log_2(e) \cdot \left( \frac{3}{8} \cdot \left( \frac{5}{9} \alpha^2 - \left(\frac{5}{9} \alpha^2 \right)^2 \right) - \frac{1}{8} \cdot \left( 3 \alpha^2 + 4 \alpha^4 + \left( 3 \alpha^2 + 4 \alpha^4 \right)^2 \right) + \frac{\alpha}{4} \cdot \left( \frac{4}{3} \alpha - (\frac{4}{3} \alpha)^2 \right) \right)\\
\in &\; \cO(\alpha^2)
\end{align*}
\end{proof}

\lowerboundthreevariablelemma*
\begin{proof}
Combine \cref{lem:lower-bound-squared-hellinger}, \cref{lem:lower-bound-project-one-to-two}, and \cref{lem:lower-bound-project-two-to-one} with $\alpha$ as $\sqrt{\eps}$.
\end{proof}

\lowerboundforthreevariables*
\begin{proof}
Our result is obtained by reducing the problem of finding an $\eps$-close graph orientation of $X - Z - Y$ to the problem of \emph{testing} whether the samples are drawn from $\pd_1$ or $\pd_2$ (sharing the same skeleton but having different orientations), with respect to the construction in \cref{lem:lower-bound-three-variable-lemma}.
To ensure $\eps$-closeness in the graph orientation, one has to correctly determine whether the samples come from $\pd_1$ or $\pd_2$ and then pick $G_1$ or $G_2$ respectively.
Put differently, if one can learn an $\eps$-close graph orientation, then one can use that algorithm to distinguish between $P_1$ and $P_2$ in \cref{lem:lower-bound-three-variable-lemma}.
However, it is well-known that distinguishing two distributions whose squared Hellinger distance is $\eps$ requires $\Omega(1/\eps)$ samples (e.g.\ see \cite[Theorem 4.7]{BarYossef:02}).
\end{proof}

\lowerboundfornvariables*
\begin{proof}
Consider a distribution $P$ on $n/3$ independent copies of the lower bound construction from \cref{lem:lower-bound-for-three-variables}, where each copy is indexed by $P_i$ for $i \in \{1, \ldots, n/3\}$.
Suppose, for a contradiction, that the algorithm draws $c n/\eps$ samples for sufficiently small $c > 0$, and manages to output $Q$ that is $\eps$-close to $P$ with probability at least $2/3$.
From \cref{lem:lower-bound-for-three-variables} with error tolerance $\Omega(\eps / n)$, we know that each copy is \emph{not} $\Omega(\eps / n)$-close with probability at least $1/5$.
By Chernoff bound, at least $\Omega(n)$ copies are \emph{not} $\Omega(\eps / n)$-close with probability at least $2/3$.
Then, by the tensorization of KL divergence, we see that $d_{\text{KL}} \left( \prod_{i=1}^{n/3} P_i \;\|\; \prod_{i=1}^{n/3} Q_i \right) = \sum_{i=1}^{n/3} d_{\text{KL}}(P_i \;\|\;  Q_i) > \Omega(\eps)$.
This contradicts the assumption that $Q$ is $\eps$-close to $P$ with probability at least $2/3$.
\end{proof}

\section{Constant calculation for \texorpdfstring{\citet[Theorem 1.3]{DBLP:journals/siamcomp/BhattacharyyaGPTV23}}{Bhattacharyya et al. (2023, Theorem 1.3)}}
\label{sec:appendix-explicit-constant}

In the original proof of the conditional mutual information tester in \citet[Theorem
1.3]{DBLP:journals/siamcomp/BhattacharyyaGPTV23}, the constants were implicit.
Here, we provide a reference proof of computing the constants.
Our proof approach mirrors that of \citet[Theorem
1.3]{DBLP:journals/siamcomp/BhattacharyyaGPTV23} except that we are explicit in the constants involved.
As the calculations are quite involved, we color coded some manipulation steps for easier equation matching and verification.
We will also use the following two results (\cref{lemma:property_f_asym} and \cref{cor:cor48BhattacharyyaGPTV23}) from \cite{DBLP:journals/siamcomp/BhattacharyyaGPTV23} without proof.

\begin{mylemma}[Lemma 4.2 of \cite{DBLP:journals/siamcomp/BhattacharyyaGPTV23}]
\label{lemma:property_f_asym}
For any $a \geq - b$ and $b \geq 0$,
\begin{equation}
\label{eq:f_with_constants}
\frac{1}{3} \min \left( \frac{a^2}{b}, | a | \log \left( 2 + \frac{|a|}{b} \right) \right) \leq f (a, b) \leq \min \left( \frac{a^2}{b}, | a | \log \left( 2 + \frac{| a |}{b} \right) \right) \;.
\end{equation}
\end{mylemma}

\begin{mycorollary}[Corollary 4.8 of \cite{DBLP:journals/siamcomp/BhattacharyyaGPTV23}]
\label{cor:cor48BhattacharyyaGPTV23}
Let $\hat{P}$ be empirical distribution over $N \geq 1$ samples.
If $P$ is a product distribution with marginals $P_x$ and $P_y$, then with probability $1 - \delta$,
\begin{equation}
I (\hat{X} ; \hat{Y} | \hat{Z} = z)
\leq \frac{1}{N} \cdot \log \left( \frac{(N + 1)^{| \Sigma |^2}}{\delta} \right) \;.
\end{equation}
\end{mycorollary}

\bigskip

\begin{claim}[c.f.\ Claim 4.3 of \cite{DBLP:journals/siamcomp/BhattacharyyaGPTV23}]
\label{claim:property_f}
For any $x, y$, the following holds:
\begin{enumerate}
  \item $f (\Delta_{x y}, P_x P_y) \leq 1$.
  \item $\min (P_x, P_y, | \Delta_{x y} |) \geq \frac{1}{2} \frac{f
  (\Delta_{x y}, P_x P_y)}{\log (6 / f (\Delta_{x y}, P_x P_y))}$.
\end{enumerate}
\end{claim}
\begin{proof}
As the first statement has no constants involved, we will focus only on the second statement.

By definition,
\begin{equation}
\label{eq:mod-delta-at-most-max-Px-or-Py}
| \Delta_{x y} |
= \max\{P_{x y} - P_x P_y, P_x P_y - P_{x y}\}
\leq \max (P_{x y}, P_x P_y)
\leq \max\{P_x, P_y\}    
\end{equation}

So, by \eqref{eq:f_with_constants}, we have
\begin{align}
f (\Delta_{x y}, P_x P_y)
& \leq | \Delta_{x y} | \cdot \log \left( 2 +
\frac{| \Delta_{x y} |}{P_x P_y} \right) \label{eq:f_and_Delta}\\
& \leq P_x \cdot \log \left( 2 + \frac{1}{P_x} \right) && \text{By \eqref{eq:mod-delta-at-most-max-Px-or-Py}} \nonumber\\
& \leq P_x \cdot \log \left( \frac{3}{P_x} \right) \;. \label{eq:f_and_Delta_lastline}
\end{align}

Taking log on both sides, we see that
\begin{eqnarray*}
& \log f (\Delta_{x y}, P_x P_y) & \leq \log P_x + \log \log (3 /
P_x)\\
\iff & \log 3 - \log f (\Delta_{x y}, P_x P_y) & \geq \log 3 - \log P_x - \log \log (3 / P_x)\\
\iff & \frac{1}{\log (3 / f (\Delta_{x y}, P_x P_y))} & \leq
\frac{1}{\log (3 / P_x) - \log \log (3 / P_x)} \;.
\end{eqnarray*}

Multiplying the inequality of \eqref{eq:f_and_Delta_lastline} in the numerator, we get
\begin{align*}
\frac{f (\Delta_{x y}, P_x P_y)}{\log (3 / f (\Delta_{x y}, P_x P_y))}
& \leq \frac{P_x \log (3 / P_x)}{\log (3 / P_x) - \log \log (3 / P_x)}\\
& = \frac{P_x}{1 - \frac{\log \log (3 / P_x)}{\log (3 / P_x)}} && \text{Divide by $\log(3/P_x)$ on top and bottom}\\
& \leq \frac{P_x}{1 - \frac{1}{e}} && \text{Since $\frac{\log x}{x} \leq \frac{1}{e}$ for $x > 0$}\\
& \leq \frac{P_x}{0.5}\\
& = 2 \cdot P_x
\end{align*}

By repeating the same argument above for $P_y$ instead of $P_x$, we get
\begin{equation}
\label{eq:min-Px-and-Py-at-least-some-f-fraction}
\min (P_x, P_y) \geq \frac{1}{2}  \frac{f (\Delta_{x y}, P_x P_y)}{\log (3 / f (\Delta_{x y}, P_x P_y))} \;.    
\end{equation}

Let us abbreviate $f$ as $f(\Delta_{x y}, P_x P_y)$ for the following calculations.
\begin{align}
| \Delta_{x y} |
& \geq \frac{f}{\log (3 / P_y)} && \text{By \eqref{eq:f_and_Delta} and \eqref{eq:mod-delta-at-most-max-Px-or-Py}} \nonumber\\
& \geq \frac{f}{\log \left( \frac{6 \log (3 / f)}{f} \right)} && \text{By \eqref{eq:min-Px-and-Py-at-least-some-f-fraction}} \nonumber\\
& = \frac{f}{\log 2 + \log \log (3 / f) + \log (3 / f)} \nonumber\\
& \geq \frac{f}{2 \log (2) + 2 \log (3 / f)} && \text{Since $x \geq \log x$} \nonumber\\
& \geq \frac{1}{2} \cdot \frac{f}{\log (6 / f)} \label{eq:mod-delta-at-least-some-f-fraction}
\end{align}

Combining \eqref{eq:min-Px-and-Py-at-least-some-f-fraction} and \eqref{eq:mod-delta-at-least-some-f-fraction}, we get $\min (P_x, P_y, | \Delta_{x y} |) \geq \frac{1}{2} 
\frac{f (\Delta_{x y}, P_x P_y)}{\log (6 / f (\Delta_{x y}, P_x P_y))}$.
\end{proof}

\begin{claim}[c.f.\ Claim 4.4 of \cite{DBLP:journals/siamcomp/BhattacharyyaGPTV23}]
\label{claim:high_prob_event}
Fix any $x \in \Sigma, y \in \Sigma$.
Let $\hat{P}_x, \hat{P}_y$, and $\hat{P}_{x y}$ be empirical estimates over $N > 1$ samples.
Then, each of the following bound holds with $1 - 3 \delta$ probability:
\begin{equation}
\label{claim:high_prob_event_1}
| \hat{P}_x - P_x | \leq 3 \left( \sqrt{\frac{P_x \log
\frac{2}{\delta}}{N}} + \frac{\log \frac{2}{\delta}}{N} \right)
\end{equation}
\begin{equation}
\label{claim:high_prob_event_2}
| \hat{P}_{x y} - P_{x y} | \leq 3 \left( \sqrt{\frac{P_{x y} \log
\frac{2}{\delta}}{N}} + \frac{\log \frac{2}{\delta}}{N} \right)
\end{equation}
\begin{equation}
\label{claim:high_prob_event_3}
| \hat{P}_x \hat{P}_y - P_x P_y | \leq 6 \sqrt{P_x P_y  \frac{\log
\frac{2}{\delta}}{N}} + 18 (P_x + P_y) \frac{\log \frac{2}{\delta}}{N} + 18
\frac{\log^2  \frac{2}{\delta}}{N^2}
\end{equation}
\begin{equation}
\label{claim:high_prob_event_4}
| \hat{\Delta}_{x y} - \Delta_{x y} | \leq 3 \sqrt{| \Delta_{x y} | 
\frac{\log \frac{2}{\delta}}{N}} + 9 \sqrt{P_x P_y  \frac{\log
\frac{2}{\delta}}{N}} + 39 \frac{\log \frac{2}{\delta}}{N} + 18 \frac{\log^2
\frac{2}{\delta}}{N^2}
\end{equation}
\end{claim}
\begin{proof}
Calculating the constant of the multiplicative Chernoff bound, we can show
that
\begin{equation}
\label{eq:chernoff_bound_with_constants}
\Pr \left[ | \hat{P}_x - P_x | > \eps P_x \right]
< 2 \exp \left( - \frac{1}{3} \cdot \min\{\eps, \eps^2\} \cdot P_x \cdot N \right) \;.
\end{equation}

Upper bounding the right hand side by $\delta$ gives
\[
\min \{\eps, \eps^2\} \geq \frac{3 \log \left( \frac{2}{\delta} \right)}{P_x \cdot N}
\iff \eps \geq \frac{3 \log \left( \frac{2}{\delta} \right)}{P_x \cdot N}
\qquad \text{and} \qquad
\eps^2 \geq \frac{3 \log \left( \frac{2}{\delta} \right)}{P_x \cdot N} \;.
\]

So,
\[
\Pr \left[ | \hat{P}_x - P_x | > \max \left\{ \frac{3 \log \frac{2}{\delta}}{P_x N}, \sqrt{\frac{3 \log \frac{2}{\delta}}{P_x N}} \right\} \cdot P_x \right] \leq \delta \;.
\]

Thus, with probability $1 - \delta$, we get \eqref{claim:high_prob_event_1}:
\[
| \hat{P}_x - P_x | \leq 3 \left( \sqrt{\frac{P_x \log \frac{2}{\delta}}{N}} + \frac{\log \frac{2}{\delta}}{N} \right) \;.
\]

Repeating the same argument for $P_{xy}$ instead of $P_x$ yields \eqref{claim:high_prob_event_2}.

As for \eqref{claim:high_prob_event_3}, observe that triangle inequality gives us
\[
| \hat{P}_x \hat{P}_y - P_x P_y |
\leq | \hat{P}_x \hat{P}_y - \hat{P}_x P_y | + | \hat{P}_x P_y
- P_x P_y |
= \hat{P}_x \cdot | \hat{P}_y - P_y | + P_y \cdot | \hat{P}_x - P_x |
\]

To simplify notation in the following, let us define $\alpha = \frac{\log \frac{2}{\delta}}{N}$.
Then, applying \eqref{claim:high_prob_event_1} for the terms $\hat{P}_x$ and $\hat{P}_y$ gives
\begin{align*}
\hat{P}_x \cdot | \hat{P}_y - P_y |
&\leq 3 \left( P_x + 3 \left( \sqrt{\frac{P_x \log \frac{2}{\delta}}{N}} + \frac{\log \frac{2}{\delta}}{N} \right) \right) \left( \sqrt{\frac{P_y \log \frac{2}{\delta}}{N}} + \frac{\log \frac{2}{\delta}}{N} \right)\\
&= 3 \left( P_x + 3 \left( \sqrt{P_x \cdot \alpha} + \alpha \right) \right) \left( \sqrt{P_y \cdot \alpha} + \alpha \right)\\
&= {\color{blue}3 P_x \left( \sqrt{P_y \cdot \alpha} + \alpha \right)} + {\color{red}9 \left( \sqrt{P_x \cdot \alpha} + \alpha \right) \left( \sqrt{P_y \cdot \alpha} + \alpha \right)}\\
&= {\color{blue}3 P_x \sqrt{P_y \cdot \alpha} + 3 P_x \cdot \alpha} + {\color{red}9 \sqrt{P_x P_y} \cdot \alpha + 9 (\sqrt{P_x} + \sqrt{P_y}) \cdot \alpha^{3/2} + 9 \alpha^2}\\
&\leq 3 P_x \sqrt{P_y \cdot \alpha} + 3 P_x \cdot \alpha + 9 \sqrt{P_x P_y} \cdot \alpha + 18 (\sqrt{P_x + P_y}) \cdot \alpha^{3/2} + 9 \alpha^2\\
&= 3 P_x \sqrt{P_y \cdot \alpha} + 3 P_x \cdot \alpha + 9 \sqrt{P_x P_y} \cdot \alpha + 18 \cdot \sqrt{(P_x + P_y) \cdot \alpha} \cdot \alpha + 9 \alpha^2\\
&\leq 3 P_x \sqrt{P_y \cdot \alpha} + 3 P_x \cdot \alpha + 9 \cdot \frac{P_x + P_y}{2} \cdot \alpha + 18 \cdot \frac{P_x + P_y + \alpha}{2} \cdot \alpha + 9 \alpha^2\\
&= 3 P_x \sqrt{P_y \cdot \alpha} + 3 P_x \cdot \alpha + \frac{27}{2} \cdot (P_x + P_y) \cdot \alpha + 18 \alpha^2
\end{align*}
where the second last inequality is because $P_x, P_y \geq 0$ implies that $\max\{P_x, P_y\} \leq P_x + P_y$ while the last inequality is due to two applications of the AM-GM inequality.
Similarly,
\[
P_y \cdot | \hat{P}_x - P_x |
\leq 3 P_y \left( \sqrt{\frac{P_x \log
\frac{2}{\delta}}{N}} + \frac{\log \frac{2}{\delta}}{N} \right)
= 3 P_y \left( \sqrt{P_x \cdot \alpha} + \alpha \right)
= 3 P_y \sqrt{P_x \cdot \alpha} + 3 P_y \cdot \alpha
\]

Combining the above chain of inequalities, we have
\begin{align*}
&\; | \hat{P}_x  \hat{P}_y - P_x P_y |\\
\leq &\; \hat{P}_x \cdot | \hat{P}_y - P_y | + P_y \cdot | \hat{P}_x - P_x |\\
\leq &\; {\color{blue}3 \cdot (P_x \sqrt{P_y} + P_y \sqrt{P_x})} \cdot \sqrt{\alpha} + \left( 3 + \frac{27}{2} \right) \cdot (P_x + P_y) \cdot \alpha + 18 \alpha^2\\
= &\; {\color{blue}3 \cdot (\sqrt{P_x} + \sqrt{P_y}) \cdot \sqrt{P_x P_y}} \cdot \sqrt{\alpha} + {\color{red}\left( 3 + \frac{27}{2} \right)} \cdot (P_x + P_y) \cdot \alpha + 18 \alpha^2\\
\leq &\; {\color{blue}6 \cdot \sqrt{P_x P_y}} \cdot \sqrt{\alpha} + {\color{red}18} \cdot (P_x + P_y) \cdot \alpha + 18 \alpha^2 && \text{Since $\sqrt{P_x}, \sqrt{P_y} \leq 1$}
\end{align*}
Replacing $\alpha$ back with $\frac{\log \frac{2}{\delta}}{N}$ yields \eqref{claim:high_prob_event_3}.

Finally, for \eqref{claim:high_prob_event_4}, we will apply the inequalities from \eqref{claim:high_prob_event_2} and \eqref{claim:high_prob_event_3} to get
\begin{align*}
&\; | \hat{\Delta}_{x y} - \Delta_{x y} |\\
\leq &\; | \hat{P}_{x y} - P_{x y} | + | \hat{P}_x \hat{P}_y - P_x P_y |\\
\leq &\; {\color{blue}3} \left( \sqrt{\frac{P_{x y} \log \frac{2}{\delta}}{N}} + {\color{blue}\frac{\log \frac{2}{\delta}}{N}} \right) + 6 \sqrt{P_x P_y} \sqrt{\frac{\log \frac{2}{\delta}}{N}} + {\color{blue}18 \cdot (P_x + P_y)} \frac{\log \frac{2}{\delta}}{N} + 18 \left( \frac{\log \frac{2}{\delta}}{N} \right)^2\\
\leq &\; 3 \sqrt{\frac{{\color{red}P_{x y}} \log \frac{2}{\delta}}{N}} + 6 \sqrt{P_x P_y} \sqrt{\frac{\log \frac{2}{\delta}}{N}} + {\color{blue}39} \frac{\log \frac{2}{\delta}}{N} + 18 \left( \frac{\log \frac{2}{\delta}}{N} \right)^2\\
= &\; 3 \sqrt{{\color{red}(\Delta_{x y} + P_x P_y)} \frac{\log \frac{2}{\delta}}{N}} + 6 \sqrt{P_x P_y \frac{\log \frac{2}{\delta}}{N}} + 39 \frac{\log \frac{2}{\delta}}{N} + 18 \frac{\log^2 \frac{2}{\delta}}{N^2}\\
\leq &\; {\color{green!50!black}3 \sqrt{(| \Delta_{x y} | + P_x P_y) \frac{\log \frac{2}{\delta}}{N}}} + 6 \sqrt{P_x P_y \frac{\log \frac{2}{\delta}}{N}} + 39 \frac{\log \frac{2}{\delta}}{N} + 18 \frac{\log^2  \frac{2}{\delta}}{N^2}\\
\leq &\; {\color{green!50!black}3 \sqrt{| \Delta_{x y} | \frac{\log \frac{2}{\delta}}{N}} + 3 \sqrt{P_x P_y \frac{\log \frac{2}{\delta}}{N}}} + 6 \sqrt{P_x P_y \frac{\log \frac{2}{\delta}}{N}} + 39 \frac{\log \frac{2}{\delta}}{N} + 18 \frac{\log^2 \frac{2}{\delta}}{N^2}\\
= &\; 3 \sqrt{| \Delta_{x y} | \frac{\log \frac{2}{\delta}}{N}} + 9 \sqrt{P_x P_y \frac{\log \frac{2}{\delta}}{N}} + 39 \frac{\log \frac{2}{\delta}}{N} + 18 \frac{\log^2  \frac{2}{\delta}}{N^2}
\end{align*}
\end{proof}

\begin{mylemma}[c.f.\ Lemma 4.5 of \cite{DBLP:journals/siamcomp/BhattacharyyaGPTV23}]
\label{lemma:when_f_is_large}
Fix any $x \in \Sigma, y \in \Sigma$, $\delta > 0$, and constant $C \geq 90000$.
Let $\hat{P}$ be the empirical distribution over $N \geq 3 > e$ samples.
If $f(\Delta_{x y}, P_x P_y) \geq C \frac{\log N}{N} \log \frac{2}{\delta}$, then
\[
\Pr \left[ f(\hat{\Delta}_{x y}, \hat{P}_x  \hat{P}_y) > \frac{1}{50} f(\Delta_{x y}, P_x P_y) \right] > 1 - 3 \delta \;.
\]
\end{mylemma}
\begin{proof}
In this proof, we have the following assumption:
\begin{equation}
\label{eq:lemma-assumption}
f(\Delta_{x y}, P_x P_y) \geq C \frac{\log N}{N} \log \frac{2}{\delta}
\end{equation}

By \cref{claim:property_f}, we have $f (\Delta_{x y}, P_x P_y) \leq 1$ and
\begin{equation}
\label{eq:min_with_constant}
\min (P_x, P_y, | \Delta_{x y} |) \geq \frac{1}{2} \frac{f (\Delta_{x
y}, P_x P_y)}{\log (6 / f (\Delta_{x y}, P_x P_y))} \geq \frac{C}{4}
\cdot \frac{1}{N} \log \frac{2}{\delta} \;,
\end{equation}
where the final inequality is due to \eqref{eq:lemma-assumption}.
Meanwhile,
\begin{align*}
\frac{f (\Delta_{x y}, P_x P_y)}{\log (6 / f (\Delta_{x y}, P_x P_y))}
& = \frac{f (\Delta_{x y}, P_x P_y)}{\log 6 - \log f (\Delta_{x y}, P_x
P_y)}\\
& \geq \frac{C \frac{\log N}{N} \log \frac{2}{\delta}}{\log (6 /
C) + \log N - \log \log N - \log \log \frac{2}{\delta}} && \text{By \eqref{eq:lemma-assumption}}\\
& \geq \frac{C \frac{\log N}{N} \log \frac{2}{\delta}}{\log (6 /
C) + \log N}\\
& \geq \frac{C \frac{\log N}{N} \log \frac{2}{\delta}}{2 \log N}
&& \text{Holds when $C \geq \frac{6}{N}$}\\
& = \frac{C}{2} \frac{\log \frac{2}{\delta}}{N}
\end{align*}

Recall that $\Delta_{x y} = P_{xy} - P_x P_y$.
We will split the remaining analysis into two cases, depending on the size of $| \Delta_{x y} |$.
\begin{description}
    \item[Case 1] (Big $| \Delta_{x y} |$): $| \Delta_{x y} | = \Delta_{x y} \geq 8 P_x P_y$
    \item[Case 2] (Small $|\Delta_{x y}|$): $- P_x P_y \leq \Delta_{x y} \leq 8 P_x P_y$
\end{description}

\textbf{Case 1}: $| \Delta_{x y} | = \Delta_{x y} \geq 8 P_x P_y$\\

\begin{align*}
&\; | \hat{\Delta}_{x y} - \Delta_{x y} |\\
\leq &\; 3 \sqrt{| \Delta_{x y} |  \frac{\log \frac{2}{\delta}}{N}} + 9 \sqrt{P_x P_y  \frac{\log \frac{2}{\delta}}{N}} + 39 \frac{\log \frac{2}{\delta}}{N} + 18 \frac{\log^2  \frac{2}{\delta}}{N^2} && \text{By \cref{claim:high_prob_event}, Eqn \eqref{claim:high_prob_event_4}}\\
\leq &\; {\color{blue}3 \sqrt{| \Delta_{x y} | \frac{4}{C} | \Delta_{x y} |}} + {\color{red}9 \sqrt{\frac{| \Delta_{x y} |}{8} \frac{4}{C} | \Delta_{x y} |}} + {\color{green!50!black}39 \frac{4}{C}  | \Delta_{x y} | + 18 \left( \frac{4}{C}  | \Delta_{x y} | \right)^2} && \text{By \eqref{eq:min_with_constant} with $|\Delta_{xy}|$}\\
\leq &\; {\color{blue}\sqrt{\frac{36}{C}} | \Delta_{x y} |} + {\color{red}\sqrt{\frac{81}{2C}} | \Delta_{x y} |} + {\color{green!50!black}\left( \frac{160}{C} + \frac{300}{C^2} \right) | \Delta_{x y} |} && \text{Since $| \Delta_{x y} |^2 \leq | \Delta_{x y} |$}\\
\leq &\; \frac{1}{10} | \Delta_{x y} | && \text{Since $C \geq 40000$}
\end{align*}

There are two ways to bound $| \hat{P}_x \hat{P}_y - P_x P_y |$ after applying \cref{claim:high_prob_event} and \eqref{eq:min_with_constant} with $|\Delta_{xy}|$.
We can either use the bound of Case 1 ($| \Delta_{x y} | = \Delta_{x y} \geq 8 P_x P_y$), or the bound of \eqref{eq:min_with_constant} with $\min\{P_x, P_y\}$.

Using $| \Delta_{x y} | = \Delta_{x y} \geq 8 P_x P_y$, we get
\begin{align*}
&\; | \hat{P}_x \hat{P}_y - P_x P_y |\\
\leq &\; 6 \sqrt{P_x P_y \frac{\log \frac{2}{\delta}}{N}} + 18 (P_x + P_y) \frac{\log \frac{2}{\delta}}{N} + 18 \frac{\log^2  \frac{2}{\delta}}{N^2} && \text{By \cref{claim:high_prob_event}, Eqn \eqref{claim:high_prob_event_3}}\\
\leq &\; 6 \sqrt{P_x P_y \frac{4}{C} | \Delta_{x y} |} + 18 \cdot \frac{4}{C} | \Delta_{x y} | + 18 \cdot \left( \frac{4}{C} | \Delta_{x y} | \right)^2 && \text{By \eqref{eq:min_with_constant} with $|\Delta_{xy}|$}\\
\leq &\; {\color{blue}6 \sqrt{\frac{1}{8} | \Delta_{x y} | \frac{4}{C} | \Delta_{x y} |}} + {\color{red}18 \cdot \frac{4}{C} | \Delta_{x y} |} + {\color{green!50!black}18 \cdot \left( \frac{4}{C} | \Delta_{x y} | \right)^2} && \text{Case 1}\\
\leq &\; {\color{blue}\sqrt{\frac{18}{C}} | \Delta_{x y} |} + {\color{red}\frac{72}{C} |
\Delta_{x y} |} + {\color{green!50!black}\frac{300}{C^2} | \Delta_{x y} |} && \text{Since $| \Delta_{x y} |^2 \leq | \Delta_{x y} |$}\\
\leq &\; \frac{1}{10}  | \Delta_{x y} | && \text{Since $C \geq 3600$}
\end{align*}

Using the bound of \eqref{eq:min_with_constant} with $\min\{P_x, P_y\}$, we get
\begin{align*}
&\; | \hat{P}_x \hat{P}_y - P_x P_y |\\
\leq &\; 6 \sqrt{P_x P_y \frac{\log \frac{2}{\delta}}{N}} + 18 (P_x + P_y) \frac{\log \frac{2}{\delta}}{N} + 18 \frac{\log^2  \frac{2}{\delta}}{N^2} && \text{By \cref{claim:high_prob_event}, Eqn \eqref{claim:high_prob_event_3}}\\
\leq &\; 6 \sqrt{P_x P_y \frac{4}{C} | \Delta_{x y} |} + 18 (P_x + P_y) \frac{\log \frac{2}{\delta}}{N} + 18 \frac{\log^2 \frac{2}{\delta}}{N^2} && \text{By \eqref{eq:min_with_constant} with $|\Delta_{xy}|$}\\
= &\; 6 \sqrt{P_x P_y \frac{4}{C} | \Delta_{x y} |} + {\color{blue}18 \left( P_x \frac{\log \frac{2}{\delta}}{N} + P_y \frac{\log \frac{2}{\delta}}{N} \right) + 18 \frac{\log^2 \frac{2}{\delta}}{N^2}}\\
\leq &\; 6 \sqrt{P_x P_y \frac{4}{C} | \Delta_{x y} |} + {\color{blue}\left( 144 \frac{1}{C} + 288 \frac{1}{C^2} \right) P_x P_y} && \text{By \eqref{eq:min_with_constant} with $\min\{P_x, P_y\}$}\\
= &\; \frac{12}{\sqrt{C}} \sqrt{P_x P_y \Delta_{x y}} + \left( 144 \frac{1}{C} + 288 \frac{1}{C^2} \right) P_x P_y
\end{align*}

Recall that in Case 1, we have $| \Delta_{x y} | = \Delta_{x y} \geq 8 P_x P_y$.
Therefore, from the above, we get:
\begin{equation}
\label{eq:first-inequality}
\hat{\Delta}_{x y} \geq \frac{9}{10} | \Delta_{x y} | = 0.9 \Delta_{x y}
\end{equation}
\begin{equation}
\label{eq:second-inequality}
\hat{P}_x \hat{P}_y \leq P_x P_y + \frac{1}{10} | \Delta_{x y} | \leq 0.23 \Delta_{x y}
\end{equation}
\begin{equation}
\label{eq:third-inequality}
\hat{P}_x \hat{P}_y \leq \frac{12}{\sqrt{C}}  \sqrt{P_x P_y | \Delta_{x y} |} + \left( 144 \frac{1}{C} + 288 \frac{1}{C^2} + 1 \right) P_x P_y
\end{equation}
The first two inequalities \eqref{eq:first-inequality} and \eqref{eq:second-inequality} together yield
\[
\frac{\hat{\Delta}_{x y}}{\hat{P}_x  \hat{P}_y}
\geq \frac{0.9}{0.23}  \frac{\Delta_{x y}}{\Delta_{x y}}
\geq 3.9
\]
Since $x \geq \log(2+x)$ for any $x \geq 2$, we see that
\begin{equation}
\label{eq:x-larger-than-log-two-plus-x}
\min \left\{ \frac{\hat{\Delta}_{x y}^2}{\hat{P}_x \hat{P}_y}, |\hat{\Delta}_{x y}| \log \left( 2 + \frac{|\hat{\Delta}_{x y}|}{\hat{P}_x \hat{P}_y} \right) \right\}
= |\hat{\Delta}_{x y}| \log \left( 2 + \frac{|\hat{\Delta}_{x y}|}{\hat{P}_x \hat{P}_y} \right)
\geq |\hat{\Delta}_{x y}| \log \frac{|\hat{\Delta}_{x y}|}{\hat{P}_x \hat{P}_y}
\end{equation}

Putting everything together,
\begin{align*}
&\; f (\hat{\Delta}_{x y}, \hat{P}_x \hat{P}_y)\\
\geq &\; \frac{1}{3} \min \left\{ \frac{\hat{\Delta}_{x y}^2}{\hat{P}_x \hat{P}_y}, |\hat{\Delta}_{x y}| \log \left( 2 + \frac{|\hat{\Delta}_{x y}|}{\hat{P}_x \hat{P}_y} \right) \right\} && \text{By \cref{lemma:property_f_asym}}\\
\geq &\; \frac{1}{3} |\hat{\Delta}_{x y}| \log \frac{|\hat{\Delta}_{x y}|}{\hat{P}_x \hat{P}_y} && \text{By \eqref{eq:x-larger-than-log-two-plus-x}}\\
\geq &\; \frac{3}{10} \Delta_{x y} \left( \log \frac{9}{10} + \log \left( \frac{\Delta_{x y}}{\hat{P}_x  \hat{P}_y} \right) \right) && \text{Since $\hat{\Delta}_{x y} \geq 0.9 \Delta_{x y}$}\\
\geq &\; \frac{3}{10} \Delta_{x y}  \left( \log \frac{9}{10} + \log \left( \frac{\sqrt{\Delta_{x y}}}{\frac{12}{\sqrt{C}} \sqrt{P_x P_y} + \left( 144 \frac{1}{C} + 288 \frac{1}{C^2} + 1 \right) {\color{blue}\frac{P_x P_y}{\sqrt{\Delta_{x y}}}}} \right) \right) && \text{By \eqref{eq:third-inequality}}\\
\geq &\; \frac{3}{10} \Delta_{x y}  \left( \log \frac{9}{10} + \log \left( \frac{\sqrt{\Delta_{x y}}}{\frac{12}{\sqrt{C}}  \sqrt{P_x P_y} + \left( 144 \frac{1}{C} + 288 \frac{1}{C^2} + 1 \right) {\color{blue}\frac{\sqrt{P_x P_y}}{\sqrt{8}} }} \right) \right) && \text{Case 1}\\
\geq &\; \frac{3}{10} \Delta_{x y}  \left( \log \left( \frac{9}{10 \sqrt{2}} {\color{red}\sqrt{\frac{\Delta_{x y}}{P_x P_y}}} \right) \right) && \text{Since $C \geq 288$}\\
\geq &\; \frac{3}{10} \Delta_{x y}  \left( \log \left( \frac{9 {\color{red}(8)^{1 / 4}}}{10 \sqrt{2}} {\color{red}\left( \frac{\Delta_{x y}}{P_x P_y} \right)^{1 / 4}} \right) \right) && \text{Case 1}\\
\geq &\; \frac{3}{40} \Delta_{x y} \log \left( \frac{\Delta_{x y}}{P_x P_y} \right) && \text{Since $\frac{9 (8)^{1 / 4}}{10 \sqrt{2}} > 1.07\ldots$}\\
\geq &\; \frac{3}{40} f(\Delta_{x y}, P_x P_y) && \text{By \cref{lemma:property_f_asym}}\\
\geq &\; \frac{1}{50} f(\Delta_{x y}, P_x P_y)
\end{align*}

\textbf{Case 2}: $- P_x P_y \leq \Delta_{x y} \leq 8 P_x P_y$\\

In this case,
\begin{align*}
C \frac{\log N}{N} \log \frac{2}{\delta}
&\; \leq f (\Delta_{x y}, P_x P_y) && \text{By \eqref{eq:lemma-assumption}}\\
&\; \leq \min \left\{ \frac{\Delta_{x y}^2}{P_x P_y}, |\Delta_{x y}| \log \left( 2 + \frac{|\Delta_{x y}|}{P_x P_y} \right) \right\} && \text{By \cref{lemma:property_f_asym}}\\
&\; \leq \frac{\Delta_{x y}^2}{P_x P_y}\\
&\; \leq 64 P_x P_y  && \text{Case 2}\\
&\; \leq 64 \min \{P_x, P_y\} && \text{Since $P_x, P_y \leq 1$}
\end{align*}

There are two observations we can make here.
Firstly, we have $C \frac{\log N}{N} \log \frac{2}{\delta} \leq \frac{\Delta_{x y}^2}{P_x P_y}$, and so
\begin{equation}
\label{eq:Delta_inequality_when_small}
\sqrt{C P_x P_y \frac{\log N}{N} \log \frac{2}{\delta}}
\leq | \Delta_{x y} |
\end{equation}
Secondly, since $N \geq 3 > e$, we see that
\[
C \frac{\log N}{N} \log \frac{2}{\delta}
\leq 64 \min \{P_x, P_y\}
\iff
\frac{\log \frac{2}{\delta}}{N}
\leq \frac{64}{C} \frac{\min \{P_x, P_y\}}{\log N}
\]
That is,
\begin{equation}
\label{eq:log-n-smaller-than-p-over-c}
\frac{\log \frac{2}{\delta}}{N}
\leq \frac{64}{C} \min \{P_x, P_y\}
\end{equation}

Combining \eqref{eq:log-n-smaller-than-p-over-c} with \eqref{claim:high_prob_event_1} gives the following:
\[
| \hat{P}_x - P_x |
\leq 3 \left( \sqrt{\frac{P_x \log \frac{2}{\delta}}{N}} + \frac{\log \frac{2}{\delta}}{N} \right)
\leq 3 \left( \frac{8}{\sqrt{C}} + \frac{64}{C} \right) P_x
\leq \frac{P_x}{10}
\]
Likewise, we have
\[
| \hat{P}_y - P_y | \leq \frac{P_y}{10}
\]

\begin{align*}
\Big| | \Delta_{x y} | - | \hat{\Delta}_{x y} | \Big|
& \leq | \hat{\Delta}_{x y} - \Delta_{x y} | && \text{Reverse triangle inequality}\\
& \leq 3 \sqrt{| \Delta_{x y} | 
\frac{\log \frac{2}{\delta}}{N}} + 9 \sqrt{P_x P_y  \frac{\log
\frac{2}{\delta}}{N}} + 39 \frac{\log \frac{2}{\delta}}{N} + 18 \frac{\log^2
\frac{2}{\delta}}{N^2} && \text{By \eqref{claim:high_prob_event_4}}\\
& \leq 3 \sqrt{| \Delta_{x y} | \frac{\log \frac{2}{\delta}}{N}} + 9 \sqrt{P_x P_y  \frac{\log \frac{2}{\delta}}{N}} + 60 \frac{\log \frac{2}{\delta}}{N} && \text{Since $\frac{\log \frac{2}{\delta}}{N} \leq 1$}\\
& \leq 3 \sqrt{| \Delta_{x y} | \frac{\log \frac{2}{\delta}}{N}} + 9 {\color{blue}\sqrt{P_x P_y \log N \frac{\log \frac{2}{\delta}}{N}}} + 60 \frac{\log \frac{2}{\delta}}{N} && \text{Since $N > e$}\\
& \leq 3 \sqrt{| \Delta_{x y} | \frac{\log \frac{2}{\delta}}{N}} + 9 {\color{blue}\frac{| \Delta_{x y} |}{\sqrt{C}}} + 60 {\color{red}\frac{\log \frac{2}{\delta}}{N}} && \text{By \eqref{eq:Delta_inequality_when_small}}\\
& \leq 3 \sqrt{\frac{8}{C}} | \Delta_{x y} | + 9 \frac{| \Delta_{x y} |}{\sqrt{C}} + 60 {\color{red}\frac{4 | \Delta_{x y} |}{C}} && \text{By \eqref{eq:min_with_constant} with $|\Delta_{xy}|$}\\
& \leq \frac{1}{10}  | \Delta_{x y} | && \text{Since $C \geq 40000$}
\end{align*}

The above bounds give us the following:
\begin{equation}
\label{eq:three-bounds-same-time}
\frac{9}{10} P_x \leq \hat{P}_x \leq \frac{11}{10} P_x
\quad \text{and} \quad
\frac{9}{10} P_y \leq \hat{P}_y \leq \frac{11}{10} P_y
\quad \text{and} \quad
\frac{9}{10}  | \Delta_{x y} | \leq | \hat{\Delta}_{x y} | \leq \frac{11}{10}  | \Delta_{x y} |
\end{equation}

So,
\[
0 \leq
\frac{|\hat{\Delta}_{x y}|}{\hat{P}_x \hat{P}_y}
\leq \frac{11/10}{(9/10)^2} \frac{| {\Delta}_{x y} |}{{P}_x {P}_y}
\leq \frac{110}{81} \cdot 8 \;.
\]
Furthermore\footnote{e.g.\ see \url{https://www.wolframalpha.com/input?i=solve+log\%282\%2Bx\%29\%3E81\%2F880x}},
$\frac{81}{880}x < \log(2+x)$ for $0 < x < \frac{110}{81} \cdot 8$.
Thus,
\begin{equation}
\label{eq:small-delta-case-take-first-term}
\frac{81}{880} \frac{\hat{\Delta}_{x y}^2}{\hat{P}_x \hat{P}_y}
< |\hat{\Delta}_{x y}| \log \left( 2 + \frac{|\hat{\Delta}_{x y}|}{\hat{P}_x \hat{P}_y} \right)
\end{equation}

Therefore,
\begin{align*}
&\; f (\hat{\Delta}_{x y}, \hat{P}_x  \hat{P}_y)\\
\geq &\; \frac{1}{3} \min \left\{ \frac{\hat{\Delta}_{x y}^2}{\hat{P}_x \hat{P}_y}, |\hat{\Delta}_{x y}| \log \left( 2 + \frac{|\hat{\Delta}_{x y}|}{\hat{P}_x \hat{P}_y} \right) \right\} && \text{By \cref{lemma:property_f_asym}}\\
\geq &\; \frac{1}{3} \cdot \frac{81}{880} \cdot {\color{blue}|\hat{\Delta}_{x y}|^2} \cdot {\color{red}\frac{1}{\hat{P}_x}} \cdot {\color{green!50!black}\frac{1}{\hat{P}_y}} && \text{By \eqref{eq:small-delta-case-take-first-term}}\\
\geq &\; \frac{27}{880} {\color{blue}\left(\frac{9}{10} | \Delta_{x y} |\right)^2} \cdot {\color{red}\frac{10}{11} \frac{1}{P_x}} \cdot {\color{green!50!black}\frac{10}{11} \frac{1}{P_y}} && \text{By \eqref{eq:three-bounds-same-time}}\\
\geq &\; \frac{1}{50} \frac{| \Delta_{x y} |^2}{P_x P_y}\\
\geq &\; \frac{1}{50} f(\Delta_{x y}, P_x P_y) && \text{By \cref{lemma:property_f_asym}}
\end{align*}

The claim holds since we showed $f (\hat{\Delta}_{x y}, \hat{P}_x  \hat{P}_y) \geq \frac{1}{50} f(\Delta_{x y}, P_x P_y)$ in both cases.
\end{proof}

\begin{mycorollary}[c.f.\ Corollary 4.6 of \cite{DBLP:journals/siamcomp/BhattacharyyaGPTV23}]
\label{corollary:when_mi_is_big}
Let $\hat{P}$ be the empirical distribution over $N > 1$ samples.
Then there exist universal constants $C_1 = \frac{1}{50}$ and $C_2 = 1800$ such that for every $\delta > 0$:
\[
I (\hat{X} ; \hat{Y}) \geq C_1 \cdot I (X ; Y) - C_2 \cdot | \Sigma |^2 \frac{\log N}{N} \log \left( \frac{6 | \Sigma |^2}{\delta} \right)
\]
with probability at least $1 - \delta$.
\end{mycorollary}
\begin{proof}
Since $f \geq 0$ on any input, whenever $f (\Delta_{x y}, P_x P_y) < 90000 \cdot \frac{\log N}{N} \log \left( \frac{6 | \Sigma |^2}{\delta} \right)$, we have
\[
f (\hat{\Delta}_{x y}, \hat{P}_x  \hat{P}_y)
\geq 0
> \frac{1}{50} \left( f (\Delta_{x y}, P_x P_y) - 90000 \cdot \frac{\log N}{N} \log \left( \frac{6 | \Sigma |^2}{\delta} \right) \right)
\]

Meanwhile, by \cref{lemma:when_f_is_large}, $f (\Delta_{x y}, P_x P_y) \geq 90000 \cdot \frac{\log N}{N} \log \left( \frac{6 | \Sigma |^2}{\delta} \right)$ with probability $1 - \frac{\delta}{| \Sigma |^2}$.
\[
f (\hat{\Delta}_{x y}, \hat{P}_x  \hat{P}_y) \geq \frac{1}{50} f(\Delta_{x y}, P_x P_y) \;.
\]

Summing up over the $|\Sigma|^2$ possible values of $x$ and $y$, and then taking a union bound over these $|\Sigma|^2$ events, we have
\[
I(\hat{X} ; \hat{Y})
\geq \frac{1}{50} \left( I (X ; Y) - 90000 \cdot \frac{\log N}{N} | \Sigma |^2 \log \left( \frac{6 | \Sigma |^2}{\delta} \right) \right) \;.
\]

That is, we have $C_1 = \frac{1}{50}$ and $C_2 = 90000 \cdot \frac{1}{50} = 1800$.
\end{proof}

\begin{theorem}[c.f.\ Theorem 1.3 of \cite{DBLP:journals/siamcomp/BhattacharyyaGPTV23}]
Let $(X, Y, Z)$ be three random variables over $\Sigma$, and $(\hat{X},
\hat{Y}, \hat{Z})$ be the empirical distribution over a size $N$ sample of
$(X, Y, Z)$.
For any
\[
N \geq 6.48 \times 10^6 \cdot \frac{| \Sigma |^3  \left( \log \left( \frac{| \Sigma |}{\eps \delta} \right) + \log (7.2 \times 10^5) \right) \cdot \log \left( \frac{12 | \Sigma |^2}{\delta} \right)}{\eps} \;,
\]
the following results hold with probability $1 - \delta$:
\begin{enumerate}
   \item If $I (X ; Y|Z) = 0$, then $I(\hat{X} ; \hat{Y} | \hat{Z}) \leq \frac{1}{400} \cdot \eps$
   \item If $I (X ; Y|Z) \geq \eps$, then $I (\hat{X} ; \hat{Y} |
\hat{Z}) \geq \frac{1}{400} \cdot I(X ; Y|Z)$
\end{enumerate}
\end{theorem}
\begin{proof}
Recall that $| \Sigma | \geq 2$ and $\eps, \delta \in (0, 1]$.
Let $N = 9 \frac{| \Sigma |^3 \log \left( \frac{| \Sigma |}{C \eps
\delta} \right) \cdot \log \left( \frac{12 | \Sigma |^3}{\delta} \right)}{C
\eps}$ where $C^{-1} = 720000$.
In particular, $N \geq 9 \frac{| \Sigma |^3 \log \left( \frac{| \Sigma |}{C \eps \delta} \right)}{C \eps}$, $C < 1/6 < 1$, and $\frac{\log (N + 1)}{N}$ monotonically decreases.

\bigskip

\textbf{Proving the first statement}

We will prove a stronger conclusion: $I(\hat{X} ; \hat{Y} | \hat{Z}) \leq C \cdot \eps$.
Since $C \leq \frac{1}{400} \cdot $, we get $I(\hat{X} ; \hat{Y} | \hat{Z}) \leq \frac{1}{400} \cdot \eps$ as desired.

If $I (X ; Y|Z) = 0$ then $I (X ; Y|Z = z) = 0$ for any $z$.
Let $N_z \geq 1$ be the count of $\hat{Z} = z$ and $\hat{P}_z = N_z / N$.
Then, \cref{cor:cor48BhattacharyyaGPTV23} tells us that, with probability at least $1 - \delta'$,
\begin{equation}
\label{eq:MI_ineq}
I (\hat{X} ; \hat{Y} | \hat{Z} = z)
\leq \frac{1}{N_z} \cdot \log \left( \frac{(N_z + 1)^{| \Sigma |^2}}{\delta'} \right) \leq \frac{| \Sigma |^2}{N_z} \cdot \log \left( \frac{N + 1}{\delta'} \right) \;.
\end{equation}

Using \eqref{eq:MI_ineq} and a union bound over $|\Sigma|$ values of $z$, the following holds with probability at least $1 - | \Sigma | \delta'$:\footnote{By definition, the empirical estimator $I(\hat{X} ; \hat{Y}) = \tmop{KL} (\hat{P}_{x y} \| \hat{P}_x \hat{P}_y) < \infty$, as $\hat{P}_x = 0$ implies $\hat{P}_{x y} = 0$.
So, if $\hat{P}_z = N_z = 0$, then we have $\hat{P}_z I (\hat{X} ; \hat{Y} | \hat{Z} = z) = 0$.}
\[
I (\hat{X} ; \hat{Y} | \hat{Z})
= \sum_z \hat{P}_z \cdot I (\hat{X} ; \hat{Y} | \hat{Z} = z)
\leq \sum_z \frac{N_z}{N} \cdot \frac{| \Sigma |^2}{N_z} \cdot \log \left( \frac{N + 1}{\delta'} \right)
= \frac{| \Sigma |^3}{N} \cdot \log \left( \frac{N + 1}{\delta'} \right) \;.
\]

Rescaling $\delta' = \frac{\delta}{| \Sigma |}$, we get $I (\hat{X} ; \hat{Y} | \hat{Z}) \leq \frac{| \Sigma |^3}{N} \log \left( \frac{(N + 1)  | \Sigma |}{\delta} \right)$.
For $I (\hat{X} ; \hat{Y} | \hat{Z})$ to be at most $C \cdot \eps$, it suffices to argue that
\[
\frac{| \Sigma |^3}{N} \log \left( \frac{(N + 1)  | \Sigma |}{\delta} \right)
\leq C \eps \;.
\]

To see this, consider
\begin{align*}
&\; \frac{| \Sigma |^3}{N} \log \left( \frac{(N + 1)  | \Sigma |}{\delta} \right)\\
= &\; \frac{| \Sigma |^3}{N}  \left( \log (N + 1) + \log \left( \frac{| \Sigma |}{\delta} \right) \right)\\
\leq &\; \frac{C \eps}{9 \log \left( \frac{| \Sigma |}{C \eps \delta} \right)}  \left( \log \left( 9 \frac{| \Sigma |^3 \log \left( \frac{| \Sigma |}{C \eps \delta} \right)}{C \eps} + 1 \right) + \log \left( \frac{| \Sigma |}{\delta} \right) \right) && \text{Since $N \geq 9 \frac{| \Sigma |^3 \log \left( \frac{| \Sigma |}{C \eps \delta} \right)}{C \eps}$}\\
= &\; \frac{C \eps \log \left( {\color{blue}9 \frac{| \Sigma |^3 \log \left( \frac{| \Sigma |}{C \eps \delta} \right)}{C \eps} + 1} \right)}{9 \log \left( \frac{| \Sigma |}{C \eps \delta} \right)} + \frac{C \eps}{9}  {\color{red}\frac{\log \left( \frac{| \Sigma |}{\delta} \right)}{\log \left( \frac{| \Sigma |}{C \eps \delta} \right)}}\\
\leq &\; \frac{C \eps \log \left( {\color{blue}10 \frac{| \Sigma |^3 \log \left( \frac{| \Sigma |}{C \eps \delta} \right)}{C \eps}} \right)}{9 \log \left( \frac{| \Sigma |}{C \eps \delta} \right)} + \frac{C \eps}{9} \cdot {\color{red}1} && \text{Since $| \Sigma | \geq 2$, $C < 1$, and $\eps, \delta \in (0, 1]$}\\
= &\; \frac{C \eps \log 10}{9 \log \left( \frac{| \Sigma |}{C \eps \delta} \right)} + \frac{C \eps \log \left( \frac{| \Sigma |^3}{C \eps} \right)}{9 \log \left( \frac{| \Sigma |}{C \eps \delta} \right)} + \frac{C \eps {\color{green!50!black}\log \left( \log \left( \frac{| \Sigma |}{C \eps \delta} \right) \right)}}{9 {\color{green!50!black}\log \left( \frac{| \Sigma |}{C \eps \delta} \right)}} + \frac{C \eps}{9}\\
\leq &\; \frac{C \eps \log 10}{9 \log \left( \frac{| \Sigma |}{C \eps \delta} \right)} + \frac{C \eps \log \left( \frac{| \Sigma |^3}{C \eps} \right)}{9 \log \left( \frac{| \Sigma |}{C \eps \delta} \right)} + \frac{C \eps}{9} \cdot {\color{green!50!black}\frac{1}{e}} + \frac{C \eps}{9} && \text{Since $\frac{\log \log x}{\log x} \leq \frac{1}{e}$}\\
= &\; C \eps {\color{blue}\frac{\log 10}{\frac{9}{3.4} \log (| \Sigma |^{3.4})}} + C \eps {\color{red}\frac{\log \left( \frac{| \Sigma |^3}{C \eps} \right)}{3 \log \left( \left( \frac{| \Sigma |}{C \eps} \right)^3 \right)}} + \frac{C \eps}{{\color{green!50!black}9}} \cdot {\color{green!50!black}\frac{1}{e}} + \frac{C \eps}{{\color{orange}9}}\\
\leq &\; C \eps \cdot \left( {\color{blue}\frac{3.4}{9}} + {\color{red}\frac{1}{3}} + {\color{green!50!black}\frac{1}{9e}} + {\color{orange}\frac{1}{9}} \right) && \text{Since $| \Sigma | \geq 2$, $C < 1$, and $\eps, \delta \in (0, 1]$}\\
\leq &\; C \eps
\end{align*}

Therefore, with probability at least $1-\delta$, $I (X ; Y|Z) = 0$ implies that $I (\hat{X} ; \hat{Y} | \hat{Z}) \leq C \cdot \eps$.

\textbf{Proving the second statement}

Consider the set $S$ of values $z \in \Sigma$ which satisfy $P_z \cdot I (X ; Y|Z = z) \geq \frac{I (X ; Y|Z)}{2 | \Sigma |} \geq \frac{\eps}{2 | \Sigma |}$.
Since $I (X ; Y|Z = z) \leq \log (| \Sigma |)$, this implies that $P_z \geq \frac{\eps}{2 | \Sigma | \log | \Sigma |}$ for $z \in S$.

We also have
\begin{align}
&\; \sum_{z \in S} P_z \cdot I (X ; Y|Z = z) \nonumber\\
= &\; \sum_{z \in \Sigma} P_z \cdot I (X ; Y|Z = z) - \sum_{z \not\in S} P_z \cdot I (X ; Y|Z = z) \nonumber\\
\geq &\; I (X ; Y|Z) - |\Sigma| \cdot \frac{I (X ; Y|Z)}{2 |\Sigma|} && \text{$z \not\in S$ $\implies$ $P_z \cdot I (X ; Y|Z = z) < \frac{\eps}{2 | \Sigma |}$}\nonumber\\
= &\;\frac{I (X ; Y|Z)}{2} \label{eq:mutual_info_S}
\end{align}

Taking a union bound over $|S| \leq |\Sigma|$ values of $z$, \cref{corollary:when_mi_is_big} states that the following holds for all $z \in S$, with probability $1 - | \Sigma | \delta'$, $C_1 = \frac{1}{50}$, and $C_2 = 1800$:
\begin{equation}
\label{eq:invoking-corollary}
I (\hat{X} ; \hat{Y} | \hat{Z} = z) 
\geq C_1 \cdot I (X ; Y|Z = z) - C_2 \cdot | \Sigma |^2 \frac{\log (N_z)}{N_z} \log \left( \frac{6 | \Sigma |^2}{\delta'} \right) \;.    
\end{equation}

Calculating the constant of the multiplicative Chernoff bound, we can show
that
\begin{equation}
\label{eq:chernoff_bound_with_constants_final}
\Pr \left[ | \hat{P}_x - P_x | > \eps P_x \right]
< 2 \exp \left( - \frac{1}{3} \cdot \min\{\eps, \eps^2\} \cdot P_x \cdot N \right) \;.
\end{equation}
Plugging in $\eps = 1/2$ in \eqref{eq:chernoff_bound_with_constants_final}, we get $\Pr [| \hat{P}_z - P_z | > P_z / 2] \leq 2 \exp \left( - \frac{1}{3} \cdot \frac{1}{4} \cdot P_z \cdot N \right)$.
Now, since $N \geq 24 \frac{| \Sigma | \log | \Sigma | \log \frac{2}{\delta'}}{\eps}$ and $P_z \geq \frac{\eps}{2 | \Sigma | \log | \Sigma |}$ for any $z \in S$, we see that $P_z \geq 12 \frac{\log (2 / \delta')}{N}$ and so the following holds for any $z \in S$:
\[
\Pr [| \hat{P}_z - P_z | > P_z / 2]
\leq 2 \exp (- \log (2 / \delta'))
= \delta' \;.
\]
Recall that $N_z = N \hat{P}_z$ is the empirical count of $\hat{Z} = z$.
So, by a union bound over $|S| \leq |\Sigma|$ values of $z$, we have that \begin{equation}
\label{eq:bound-on-N_z}
N_z = N \hat{P}_z \geq N P_z / 2
\end{equation}
for all $z \in S$, with probability $1 - | \Sigma | \delta'$.

Notice that we performed union bound twice, once in \eqref{eq:invoking-corollary} and once in \eqref{eq:bound-on-N_z}.
Combining these expressions, we get that at least with probability $1 - 2 | \Sigma | \delta'$,
\begin{equation}
\label{eq:combined-double-union-bound}
I (\hat{X} ; \hat{Y} | \hat{Z} = z)
\geq C_1 \cdot I (X ; Y|Z = z) - C_2 \cdot | \Sigma |^2  \frac{\log (N P_z / 2)}{N P_z / 2} \log \left( \frac{6 | \Sigma |^2}{\delta'} \right) \;.    
\end{equation}

Set $\delta = 2 |\Sigma| \delta'$ so that everything holds with probability $1 - \delta$.
Multiplying $\hat{P}_z$ and summing over $z \in S$ gives:
\begin{align*}
&\; I(\hat{X} ; \hat{Y} | \hat{Z})\\
\geq &\; \sum_{z \in S} {\color{blue}\hat{P}_z} \cdot {\color{red}I(\hat{X} ; \hat{Y} | \hat{Z} = z)} && \text{Since $S \subseteq \Sigma$}\\
\geq &\; \sum_{z \in S} {\color{blue}\frac{P_z}{2}} {\color{red}\left( C_1 \cdot I (X ; Y|Z = z) - C_2 \cdot | \Sigma |^2  \frac{\log (N P_z / 2)}{N P_z / 2} \log \left( \frac{6 | \Sigma |^2}{\delta'} \right) \right)} && \text{By \eqref{eq:bound-on-N_z} and \eqref{eq:combined-double-union-bound}}\\
= &\; \sum_{z \in S} \frac{1}{2} \left( C_1 \cdot P_z \cdot I (X ; Y|Z = z) - 2 \cdot C_2 \cdot | \Sigma |^2 \frac{\log ({\color{green!50!black}N P_z / 2})}{N} \log \left( \frac{6 | \Sigma |^2}{\delta'} \right) \right)\\
\geq &\; \frac{C_1}{4} \cdot I (X ; Y|Z) - {\color{orange}\sum_{z \in S}} C_2 \cdot {\color{orange}| \Sigma |^2} \frac{\log ({\color{green!50!black}N / 2})}{N} \log \left( \frac{6 | \Sigma |^2}{\delta'} \right) && \text{By \eqref{eq:mutual_info_S} and $P_z \leq 1$}\\
\geq &\; \frac{C_1}{4}  \cdot I (X ; Y|Z) - C_2 \cdot {\color{orange}| \Sigma |^3} \frac{\log (N / 2)}{N} \log \left( \frac{6 | \Sigma |^2}{\delta'} \right) && \text{Since $S \subseteq \Sigma$}\\
= &\; \frac{C_1}{4}  \cdot I (X ; Y|Z) - C_2 \cdot | \Sigma |^3 \frac{\log (N / 2)}{N} \log \left( \frac{12 | \Sigma |^3}{\delta} \right) && \text{Since $\delta = 2 |\Sigma| \delta'$}\\
\end{align*}

Recall that $N = 9 \frac{| \Sigma |^3 \log \left( \frac{| \Sigma |}{C \eps \delta} \right) \cdot \log \left( \frac{12 | \Sigma |^3}{\delta} \right)}{C \eps}$ and $C < 1/6$.
Continuing from above,
\begin{align*}
&\; I(\hat{X} ; \hat{Y} | \hat{Z})\\
\geq &\; \frac{C_1}{4}  \cdot I (X ; Y|Z) - C_2 \cdot | \Sigma |^3 \frac{\log (N / 2)}{N} \log \left( \frac{12 | \Sigma |^3}{\delta} \right) && \text{From above}\\
= &\; \frac{C_1}{4} \cdot I (X ; Y|Z) - C_2 \cdot C \eps \frac{\log \left( \frac{9}{2} \frac{| \Sigma |^3}{C \eps} \log \left( \frac{| \Sigma |}{C \eps \delta} \right) \log \left( \frac{12 | \Sigma |^3}{\delta} \right) \right)}{9 \log \left( \frac{| \Sigma |}{C \eps \delta} \right)} && \text{Definition of $N$}\\
= &\; \frac{C_1}{4} \cdot I (X ; Y|Z)\\
&\; \qquad - C_2 \cdot C \eps \left( {\color{blue}\frac{\log \left( \frac{9}{2}  \frac{| \Sigma |^3}{C \eps} \right)}{3 \log \left( \left( \frac{| \Sigma |}{C \eps \delta} \right)^3 \right)}} + {\color{red}\frac{\log \log \left( \frac{| \Sigma |}{C \eps \delta} \right)}{9 \log \left( \frac{| \Sigma |}{C \eps \delta} \right)}} + {\color{green!50!black}\frac{\log \log \left( \frac{12 | \Sigma |^3}{\delta} \right)}{3 \log \left( \left( \frac{| \Sigma |}{C \eps \delta} \right)^3 \right)}} \right)\\
\geq &\; \frac{C_1}{4} \cdot I (X ; Y|Z) - C_2 \cdot C \eps \left( {\color{blue}\frac{1}{3}} + {\color{red}\frac{1}{9}} + {\color{green!50!black}\frac{1}{3}} \right) && (\dag)\\
\geq &\; \frac{C_1}{4} \cdot I (X ; Y|Z) - C_2 \cdot C \eps\\
= &\; \frac{1}{4} \cdot \frac{1}{50} \cdot I (X ; Y|Z) - 1800 C \eps && \text{Since $C_1 = \frac{1}{50}$, $C_2 = 1800$}\\
\geq &\; \frac{1}{200} \eps - 1800 C \eps && \text{Since $I (X ; Y|Z) \geq \eps$}\\
= &\; \frac{1}{400} \eps && \text{Since $C^{-1} = 720000$}
\end{align*}
where $(\dag)$ is because $| \Sigma | \geq 2$, $C^{-1} = 720000$, and $\eps, \delta \in (0, 1]$.

Therefore, with probability at least $1-\delta$, $I (X ; Y|Z) \geq \eps$ implies that $I(\hat{X} ; \hat{Y} | \hat{Z}) \geq \frac{1}{400} \eps$.

\end{proof}

\end{document}